\documentclass{llncs}
\usepackage[german,english]{babel}
\usepackage{times,theorem}
\usepackage{latexsym,color,graphics}

\usepackage{epsfig}
\usepackage{amssymb}
\usepackage{amsmath}
\usepackage{amsfonts}
\usepackage{xspace}
\usepackage{graphicx}
\begingroup
\catcode`\~=11
\gdef\urltilde{\lower 0.6ex\hbox{~}}
\endgroup

 \newcommand{\B}{\mathcal{B}}
 \newcommand{\D}{\mathcal{D}}
\newcommand{\E}{\mathcal{E}} 
 
\newcommand{\I}{\mathcal{I}} 
\newcommand{\K}{\mathcal{K}} \renewcommand{\L}{\mathcal{L}}
 
 \renewcommand{\P}{\mathcal{P}}
 
 \newcommand{\T}{\mathcal{T}}
 \newcommand{\V}{\mathcal{V}}
\newcommand{\W}{\mathcal{W}} 
 \newcommand{\Z}{\mathcal{Z}}

\title{Probabilistic Extension of Neuro-Symbolic AGI Robots  based on Belnap's Many-sorted Intensional FOL}
\author{Zoran Majki\'c}
\authorrunning{Zoran Majki\'c}
\institute{ISRST, Tallahassee, FL, USA\\
\email{majk.1234@yahoo.com}}

\newtheorem{coro}{Corollary}

\begin{document}
\maketitle              

\begin{abstract}
Neuro-symbolic AI based on $IFOL_B$  is a way to combine neural learning and symbolic reasoning to overcome limitations of purely neural systems (like lack of interpretability and logical structure) with formal logical machinery for self-reference \cite{Majk26}.\\
In this paper we expand the cognitive power of $IFOL_B$ by using the probability computation for the currently unknown sentences, based on Nilsson's probability structure for the $IFOL_B$.  We introduce the global symmetry transformation that preserves the current knowledge database $\K$ and logical deduction, and the local one used for real-time decisions about concrete (sub)problems that involve only a very strict subset of $IFOL_B$ predicates.
The computation of probability density function $KI$ in both cases, based on the Shannon's maximum information entropy, is provided by neural networks of this probabilistic neuro-symbolic AGI.
\end{abstract}

\section{Introduction to IFOL-based Proposal of Self-awareness}

This theory on  robot self-awareness is rooted in the development of Intensional First-Order Logic (IFOL) \cite{Majk22}. We argue that true "Strong AI" or "autoepistemic" robots require a symbolic architecture that allows them to reason about their own internal states as distinct from the external objects they perceive.  This introduction is a short presentation of the my approach to AGI (Strong-AI) for a new generation of intelligent robots, recently published in the papers \cite{Majk23r} and \cite{Majk24a}. Neuro-symbolic AI attempts to integrate neural and symbolic architectures in a manner that addresses strengths and weaknesses of each, in a complementary fashion, in order to support robust strong AI capable of reasoning, learning, and cognitive modeling. In this approach to AGI I considered the Intensional First Order Logic  \cite{Majk22} as a symbolic architecture of modern robots, able to use natural languages to communicate with humans and to reason about their own knowledge with self-reference and abstraction language property.
In what follows we will consider the 4-valued typed $IFOL_B$ based on Belnap's bilattice of truth-values $X = \B_4 = \{f,t,\bot,\top\}$ introduced in \cite{Majk25,Majk26} and denoted by $IFOL_B$.

Intensional entities (or  concepts) are such things as Propositions,
Relations and Properties (PRP). What make them "intensional" is that they
violate the principle of extensionality; the principle that
extensional equivalence implies identity. All (or most) of these
intensional entities have been classified at one time or another as
kinds of \emph{Universals} \cite{Beal93}, in the case of many-valued $IFOL_B$,  $D_I = D_1+D_2+D_3+...$ (with  propositions (L-concepts) $D_1$, and relational concepts $D_n$, $n\geq 2$), and \emph{particulars} $D_{0}$ \cite{Beal79}, which define the PRP domain $\D = D_0+D_I$, with intensional mapping from the set of FOL formulae $\L$ to these intensional concepts
$$I:\L \rightarrow \D$$
If $\phi(\textbf{x})$ is an open formula (virtual predicate with a list (a tuple) of free variables in $\textbf{x} =(x_1,...,x_n)$), then $I(\phi(\textbf{x})) \in D_n$ is a n-ary concept. This concept mapping can be extended  to the homomorphism  $I:\mathcal{A}\mathfrak{B}_{FOL}\rightarrow \mathcal{A}\mathfrak{B}_{int}$, between the FOL syntax algebra to the algebra of concepts. Thus, for a given 4-valued Herbrand base $H$ of $IFOL_B$ and $X =\B_4 = \{f,t,\bot,\top\}$ the set of Belnap's truth-values (false, true, unknown and inconsistent, respectively) and a Herbrand interpretation that must satisfy all built-in predicates as well)
 \begin{equation} \label{eq:Herbrand}
 v:H\rightarrow X
  \end{equation}
 and its unique extension to all sentences $v^*:\L_0 \rightarrow X$,  with $v^* \in \I_{MV}$ of the set of all well-defined Herbrand interpretations (that respects the built-in predicates in Definition 4 in \cite{Majk25}). Note that the set of all well-defined Herbrand interpretations  $\I_{H}$ is only a subset of functions in $X^H$, because we have the built-in predicates as well (from Definition 4 in \cite{Majk25}),
 \begin{equation} \label{eq:Herbrand2}
 \I_{H} \subset X^H  ~~~with ~~~\I_{MV} = \{v^*:\L_0 \rightarrow X | v \in \I_H\}
 \end{equation}
 where $v^*$ is the extension of Herbrand interpretation $v$ to all sentences $\L_0$ of our typed $IFOL_B$.

 Each extensional interpretation $h$ assigns to the intensional elements of $\D$
an appropriate extension: in the case of particulars $u \in  D_{0}$, $h_0(u) \in D_0$, such that for each logic value $a\in X\subset D_0$, $h_0(a) = a$.  Thus, we have the particular's mapping $h_0:D_0\rightarrow D_0$ and more generally  (here $'+'$ is considered as disjoint union) from \cite{Majk25},
 \begin{equation}\label{eq:dueM6}
 h =   \sum_{i\in \mathbb{N}}h_i:\D \longrightarrow   D_0+\sum_{i\geq 1}\mathfrak{Rm}_i
\end{equation}
where $\mathfrak{Rm}_i$ are i-ary relations and
 where $h_1$ assigns to each  L-concept $u \in D_1$  (of a sentence with truth-value $a\in X$), a
relation composed by the single tuple $h_1(u) = \{a\}$ if $a \neq \bot$),  $~\emptyset$ otherwise, and $h_i:D_i\rightarrow \mathfrak{Rm}_i$, for $i\geq 2$, that assigns a m-extension to non-sentence concepts (obtained, for example, an $(i-1)$-ary predicate, so that the last $i$-th column of $\mathfrak{Rm}_i$ is a truth-value of a ground atom of this predicate. In this way, the relation $\mathfrak{Rm}_i$ represents the set of tuples of ground atoms of a given predicate both with their truth-values $a \in X$.

 The  two-steps interpretation of $IFOL_B$ based on two homomorphisms, fixed intensional $I$, and extensionalization mapping $h$, is provided by commutative diagram  in Corollary 3 in \cite{Majk25}, which  defines the MV-interpretation
 $$I^*_{B} = h\circ I$$.

 So we can define the following set of different "possible worlds" for this many-valued $IFOL_B$ \cite{Majk25} based on Herbrand interpretations (\ref{eq:Herbrand}) above:
\begin{equation}\label{eq:SetExtens}
\W = \{I^*_{B} 
~|~ v^* \in \I_{MV}\}~~
\end{equation}
such that for each Herbrand interpretation $v\in \I_{H} \subset X^H$, we have a unique MV-interpretation  $I^*_{B} = h\circ I$, that is, the bijections
\begin{equation}\label{eq:SetExtens2}
is_H:\I_{H} ~\simeq \W ~~~~and~~~~ is_{MV}:\I_{MV}~\simeq \W
\end{equation}
with $I^*_{B} = is_{MV}(v^*) = is_H(v)$ and $v^* = is^{-1}_{MV}\circ is_H(v)$, that is, from the fact that $I$ is fixed intensional interpretation,  each possible world is fundamentally an extensionalization function $h$, and we denote by $\E_{in}$ the set of these extensionalization functions that respects all built-in predicates, with bijection \cite{Majk25}
\begin{equation}\label{eq:SetExtens3}
is_{in}:\W\simeq \E_{in}
\end{equation}
 The extensions of the IFOL concepts change in time (the robot's knowledge), so that we can use for specification of $h$ the time-index  as their ordered representation.

 In reflective languages, reification data is causally connected to the related \index{reification} reified aspect such that a modification to one of them affects the other; by using intensional FOL the robots can formalize also the natural language expressions "I see the blue color" by a predicate "See(\textbf{I},blue color)" where the sense of the ground term "\textbf{I}" (\emph{Self}, me)\footnote{Self in a sense which implies that all our activities are controlled by powerful creatures inside ourselves, who do our thinking and feeling for us.} for a robot is the name of the  main working coordination program which activate all other algorithms (neuro-symbolic AI subprograms) like visual recognition of color of the object in focus. But also the auto-conscience sentence like "I know that I see the blue color" by using abstracting operators "$\lessdot\_\gtrdot$" of intensional FOL, expressed by the predicate "Know(\textbf{I},$\lessdot$ See(\textbf{I}, blue color)$\gtrdot$)", etc...
If $\phi(\textbf{x})$ is a virtual predicate with a list (a tuple) of free variables in $\textbf{x} =(x_1,...,x_n)$  and  $\alpha$ is its subset of \emph{distinct} variables,  then $\lessdot \phi(\textbf{x}) \gtrdot_{\alpha}^{\beta}$ is a term, where $\beta$ is the remaining set of free variables  in $\textbf{x}$. The externally quantifiable variables are the \emph{free} variables not in $\alpha$. When $n =0,~ \lessdot \phi \gtrdot$ is a term which denotes a proposition, for $n \geq 1$ it denotes  a n-ary concept.
 \begin{definition} \label{def:abstrConv} \textsc{Intensional abstraction convention}:

 From the fact that we can use any permutation of the variables in a given virtual predicate,  we introduce the convention that
 \begin{equation}\label{eq:abstrctConv}
 \lessdot \phi(\textbf{x})\gtrdot_{\alpha}^{\beta}~~ is~a~ term~ obtained~ from~ virtual ~ predicate ~~\phi(\textbf{x})
 \end{equation}
 if $\alpha$ is \textsl{not empty}   such that  $\alpha\bigcup\beta$ is  the set of all variables in the list (tuple of variables)  $\textbf{x} = (x_1,...,x_n)$ of the virtual predicate (an open logic formula) $\phi$,  and $\alpha\bigcap\beta = \emptyset$, so that $|\alpha|+|\beta| = |\textbf{x}| = n$.
 Only the variables in $\beta$ (which are the only free variables of this term), can be quantified. If $\beta$ is empty then $\lessdot \phi(\textbf{x})\gtrdot_{\alpha}$ is a \emph{ground term}. If $\phi$ is a sentence and hence both $\alpha$ and $\beta$ are empty, we write simply $\lessdot \phi \gtrdot$ for this ground term.
 \end{definition}
 More about this general definition of abstract terms can be find in \cite{Majk22}. In this paper we will use the most simple cases of ground terms $\lessdot \phi \gtrdot$, where $\phi$ is a sentence.

 By using the intensional mapping $I$, we are able to extend the simple assignment to variables to all (also abstracted) terms:
\begin{definition}
 An assignment $g:\V \rightarrow \D$ forb variables in $\V$ is applied only to free variables in terms and formulae.  Such an assignment $g \in \D^{\V}$ can be recursively uniquely extended into the assignment $g^*:\T \rightarrow \D$, where $\T$ denotes the set of all terms (here $I$ is an intensional interpretation of this FOL, as explained in what follows), by :
\begin{enumerate}
  \item $g^*(t) = g(x) \in \D$ if the term $t$ is a variable $x \in\V$.
  \item $g^*(t) = I(c) \in \D$ if the term $t$ is a constant (nullary functional symbol) $c\in P$.
  \item If $t$ is an abstracted term obtained for an open formula $\phi_i$, $\lessdot \phi_i(\textbf{x}_i) \gtrdot_{\alpha_i}^{\beta_i}$,  then we must restrict the assignment to $g\in \D^{\beta_i}$ and to obtain recursive definition (when also $\phi_i(\textbf{x}_i)$ contains abstracted terms:
\begin{equation} \label{eq:assAbTerm}
  g^*(\lessdot \phi_i(\textbf{x}_i)\gtrdot_{\alpha_i}^{\beta_i}) =_{def}
    \left\{
    \begin{array}{ll}
   I(\phi_i(\textbf{x}_i))~~ \in D_{|\alpha_i|+1}, & \hbox{if  $\beta_i$ is  empty}\\
       I(\phi_i(\textbf{x}_i)[\beta_i
/g(\beta_i)])~~ \in D_{|\alpha_i|+1}, & \hbox{otherwise}
       \end{array}
  \right.
 \end{equation}
where $g(\beta) = g(\{y_1,..,y_m\}) = \{g(y_1),...,g(y_m)\}$ and $[\beta
/g(\beta)]$ is a uniform replacement of each i-th variable in the
set $\beta$ with the i-th constant in the set $g(\beta)$. Notice that $\alpha$ is the set of all free variables in the formula $\phi[\beta /g(\beta)]$.
\item  If $~t = \lessdot \phi_i\gtrdot$ is an abstracted term obtained from a sentence $\phi_i$ then \\$g^*(\lessdot \phi_i\gtrdot) = I(\phi_i) \in D_0$.\footnote{ This case 4 is the particular case 3 when the tuple of variables $\textbf{x}_i$ is empty and hence  $\beta_i$ and $\alpha_i$ are empty sets of variables with $|\alpha_i| = 0$.}
\end{enumerate}
\end{definition}
By introduction  of the abstraction operators with autoepistemic capacities, suported by the $Know$ (meta)predicate, we do not use more a \emph{pure} logical deduction of the standard FOL, but a kind of autoepistemic deduction \cite{Majk04ph,MajkA04} with a proper set of new axioms for $IFOL_B$ provided in \cite{Majk25}.
It has been demonstrated that in such a minimal intensional enrichment of standard (extensional) FOL, we obtain exactly the Montague's definition of the intension (see Proposition 5 in \cite{Majk22}).

We recall that each robot's extensionalitation function $h$ in (\ref{eq:dueM6}) is indexed by the time-instance.
 Clearly, the robots knowledge changes in time and hence determines the extensionalization function $h$ in any given instance of time, based on robots experiences. Thus, as for humans, also the robot's knowledge and logic is a kind of temporal logic, and evolves with time.
  Note that the explicit (conscious) robot's  knowledge in actual world $\hbar$ 
   here is represented by the ground atoms of the $Know$ predicate, for a given assignments  of variables in $\V$, $g:\V\rightarrow \D$,
 \begin{equation} \label{eq:esem2}
  Know(y_1,y_2,\lessdot \psi(\textbf{x})\gtrdot^\beta_\alpha)/g = Know(g^*(y_1),g^*(y_2), g^*(\lessdot \psi(\textbf{x})\gtrdot^\beta_\alpha))
  \end{equation}
  with  $\{y_1,y_2\}\bigcup \beta \bigcup \alpha \subseteq \V$,  such that $g^*(y_1) = in ~ present$ and $g^*(y_2) = \textbf{I}$ (the robot itself),  for the extended assignments $g^*:\T \rightarrow \D$ (from Definition 17 in \cite{Majk22}),  where the set of terms $\T$ of $IFOL_B$ is composed by the set $\V$ of  variables used in the set of predicates of $IFOL_B$, by the set of constants and  \emph{abstracted terms} $\lessdot \psi(\textbf{x})\gtrdot^\beta_\alpha)$, so that in the actual world $\hbar$, the known fact (\ref{eq:esem2}) for robot becomes

  $Know(y_1,y_2,\lessdot \psi(\textbf{x})\gtrdot^\beta_\alpha)/g = Know(in ~ present,\textbf{I}, I(\psi[\beta/g(\beta)]))$\\
  which is true in actual word, that is, from proposition (intensional concept)\\ $u=I(Know(in ~ present,\textbf{I}, I(\psi[\beta/g(\beta)])))\in D_0$, we obtain the truth value \\$\hbar(u) = \hbar(I(Know(in ~ present,\textbf{I}, I(\psi[\beta/g(\beta)])))) = \{t\}$.
   Note that for the assignments $g:\V\rightarrow \D$, such that $g(y_1)= in ~future$ and $g(y_2)$ we consider robot's hypothetical knowledge in future, while in the cases when $g(y_1) = in ~past$ we consider what was robot's knowledge in the past.
  So, the robots current knowledge (ground atoms of predicate $Know$) is directly derived from its experiences (based on its neuro-system processes that robot is using) in an analog way as human brain does:
  \begin{itemize}
    \item As an activation (under robot's attention) of its neuro-system process, as a consequence of some human command to execute some particular job.
    \item As an activation of some process under current attention of robot, which is part of some complex plan of robot's activities connected with its general objectives and services.
  \end{itemize}
\textbf{Remark}:  We consider that only robot's experiences (under robot's attention) are transformed into the ground atoms of the many-valued $Know$ predicate, and the required (by robot) deductions from them (by using general many-valued deduction \cite{Majk26} extended by the three epistemic axioms) are transformed into ground atoms of $Know$ predicate, and hence are saved in robot's temporary memory as a part of robot's \emph{conscience}. Some background process (unconscious for the robot) would successively transform these temporary memory knowledge into permanent robot's knowledge
as it happen for humans. \\By such fixing by humans of robot's unconciseness part with active semantics (which can not be modified by robots and their live experience) of all significant for human robot's concepts and their properties, we will obtain ethically confident and socially safe and non danger robots (controlled by public human  ethical security organizations for the production of robots with general strong-AI capabilities).
   \\$\square$\\
IFOL-based approach is part of a general neuro-symbolic AI paradigm, where researchers aim to blend: \textbf{Neural methods} (e.g., deep learning) for pattern recognition and learning from data, and \textbf{symbolic logic} for structured knowledge representation and reasoning. This broader field (which includes work at IBM Research, MIT, and in academic surveys) seeks to build AI systems that can both \emph{learn from experience} and \emph{reason abstractly}. These ideas  contribute to ongoing discussions in AI about \emph{symbol grounding}, \emph{logical inference}, and \emph{self-referential reasoning} — all important for Strong AI research.

Key aspects of this theory include:
\begin{enumerate}
  \item \textbf{Formalizing the "Self" }(\textbf{I}): \\We propose that for a robot, the ground term "\textbf{I}" (or "me", Self) acts as the name of its main working coordination program. This master program activates all other sub-programs, such as visual recognition or motor control, allowing the robot to represent expressions like $See(\textbf{I}, blue color)$ in its logical framework.\\
  \item \textbf{Neuro-Symbolic Integration}:\\ We advocate for a "dual-process" model inspired by Daniel Kahneman’s System 1 and System 2.\\
  - \textbf{System 1 (Neural)}: Handles fast, unconscious pattern recognition (e.g., deep learning).\\
  - \textbf{System 2 (Symbolic)}: Uses IFOL (Intensional FOL) for slow, explicit, and deliberative thinking, such as planning and logical deduction.\\
  \item \textbf{Intensional Abstraction and Self-Reference}:\\ Using intensional abstraction, a robot can treat its own knowledge (propositions) as "individuals" within its logic. This allows the robot to perform self-referential reasoning—literally thinking about its own thoughts—which we identify as a prerequisite for self-awareness.\\
  \item \textbf{Grounding through Experience}:\\ We argue that robots must "ground" their language concepts by associating them with their own sensory-motor experiences. For example, the sense of the word "blue" isn't just a label, but is grounded in the robot's specific internal neural experience of processing that color.
\end{enumerate}
We addresses \textbf{logical omniscience}—the unrealistic assumption that an agent automatically knows all logical consequences of its beliefs—by using \textbf{intensional abstraction} to shift from "extensional" to "intensional" reasoning. This work is detailed in the book, Intensional First-Order Logic \cite{Majk22}: From AI to New SQL Big Data, which outlines these principles as a path toward a new generation of Strong-AI robots.
In standard modal logics (like S5), knowledge is closed under logical implication, meaning if an agent knows $A$, it must instantly know all $B$ where $A\rightarrow B$. We solve this by re-engineering how propositions are represented: 
\begin{itemize}
  \item \textbf{Reification of Propositions}: Through intensional abstraction, predicates and sentences are "reified"—treated as individual objects (intensional entities) within the same domain as physical objects.
  \item \textbf{Decoupling Truth from Meaning}: In this framework, two propositions can be extensionally equivalent (true in all the same possible worlds) but intensionally distinct. For example, a robot might know "Triangle A is equilateral" without yet knowing "Triangle A is equiangular," because those two concepts are different intensional "individuals" that require a specific computational inference step to link.
\end{itemize}
This approach allows for Strong-AI robots that can reason about their own knowledge (autoepistemic reasoning) as a finite, step-by-step process, mirroring human cognitive limitations rather than possessing infinite mathematical foresight.

In our framework, autoepistemic reasoning is the ability of a robot to reason about its own state of knowledge and belief as if they were objects in the world. Unlike standard AI, which might "know" a fact without "knowing that it knows," our Strong-AI Autoepistemic Robots use specialized logical structures to achieve formal self-reflection.
\begin{enumerate}
  \item \textbf{The "Temporal Know" Predicate}: \\We  implement autoepistemic capabilities by introducing a specific temporal "\emph{Know}" predicate. This allows the robot to handle complex internal axioms:\\
      -  \textbf{Reflexive Axiom}: The robot understands that if it knows something, that thing must be part of its internal truth model.\\
      - \textbf{Positive Introspection}: The robot can derive that "I know that I know $A$".\\
      - \textbf{Distributive Axiom}: The robot can apply its knowledge across logical implications (e.g., if it knows $A$ and knows $A\Rightarrow B$, it can reason its way to $B$). \\
  \item \textbf{Grounding through Neuronal Experience}:\\
A central part of our theory is that a robot's "internal" language must be grounded in its own hardware experiences:\\
- \textbf{Mining the "Sense"}: The robot associates high-level logical concepts with the specific firing patterns of its neural architectures.\\
- \textbf{Subjective Knowledge}: Because this grounding is unique to the robot’s own sensors and processors, its autoepistemic reasoning is truly "self-centered"—it reasons based on how it specifically perceives the world.
\\
  \item \textbf{Handling Inconsistency}:\\
Traditional logic often breaks down when faced with a contradiction (the "explosion principle"). Our Autoepistemic Logic \cite{Majk25,Majk26} is designed to be many-valued, based on Belnap's bilattice, meaning the robot can:\\
- Reason with incomplete or inconsistent information without crashing.\\
- Revise its beliefs when new "experiences" from its neural layer contradict its previous symbolic "knowledge".\\
- Support the Knowledge Assumption Closure.\\
\item \textbf{Language as a Tool for Self-Reference}:\\
We argue that natural language is inherently "many-sorted" and intensional. By giving robots an IFOL architecture, they don't just process strings of text; they use language as a symbolic coordination tool to label and manipulate their own internal programs and knowledge states
\end{enumerate}
In this research, bridging the gap between neural networks (sub-symbolic) and  symbolic logic  is not just about making them work side-by-side; it is about creating a formal translation layer where the robot's "internal feelings" (neuron firings) become "meaningful concepts" (logical terms).
We achieve this via a Dual-System Architecture grounded in Intensional First-Order Logic:
\begin{enumerate}
  \item \textbf{System 1: The "Neuronal Grounding" Layer};\\
System 1 consists of deep learning and neural networks. For us, this layer is responsible for pattern recognition and sensory-motor coordination.\\
- \textbf{The Output}: Instead of just outputting a label like "Blue," the neural network produces a specific internal state (a vector or firing pattern).\\
- \textbf{The Problem}: On its own, the neural network doesn't "understand" the concept; it just reacts. This is where the gap exists.
\\
  \item \textbf{The Bridge: Intensional Abstraction}:\\
This is our "secret sauce." We use intensional abstraction to turn the complex activity of the neural network into a "Logic Object."\\
- \textbf{Mapping}: Each distinct neural pattern is mapped to a specific intensional entity in the IFOL.\\
- \textbf{Meaning vs. Reference}: The "meaning" (intension) of a word like heavy is the specific neural experience of the robot's motors straining. The "reference" (extension) is the actual physical object being lifted.\\
- \textbf{Self-Labeling}: The robot uses autoepistemic reasoning to say, "I am currently experiencing neural state Y, which I have mapped to the concept 'Heavy'."
\\
  \item \textbf{System 2: The "Autoepistemic Logic" Layer}:\\
Once the neural patterns are abstracted into logical symbols, System 2 takes over. This is the deliberative, slow-thinking part of the brain.\\
- \textbf{High-Level Planning}: Because these symbols are now part of a formal logic $IFOL_B$, the robot can use them, both with probabilistic computations, in complex "If-Then" scenarios that neural networks struggle with.\\
- \textbf{Predefined set of axioms and prohibitions}: Any kind of reasoning results and plans must satisfy these axioms before the executions of actions derived from such plans. In this way the robots can not become danger for the humans and the environment in which they are active.\\
- \textbf{Feedback Loop}: System 2 can "query" System 1. For example, the symbolic layer might ask, "Check the visual sensors again; does that object match the 'Bird' intension?" This forces the neural network to re-process data based on logical needs.
\\
\item \textbf{Solving the "Black Box" Problem}:\\
One of the most interesting aspects of this bridge is Explainability.\\
- In standard neural networks, we don't know why a robot chose an action.\\
- In this framework, because every neural state is linked to an intensional logic term, the robot can provide a symbolic trace of its reasoning: "I performed Action A because my neural sensors triggered the 'Danger' intension, which my logic defines as a state to be avoided".
\end{enumerate}
\textbf{Summary Table:} The neuro-symbolic's Bridge\\
\begin{tabular}{|c|c|c|c|}
  \hline
    & \textbf{Sistem 1 (Neural)} & \textbf{The Bridge (Intensional)} & \textbf{System 2 (Symbolic)} \\
    \hline
  \textbf{Function} & 	Fast, reactive sensing & Abstraction $\&$ Grounding & Slow, logical planning \\
  \hline
  \textbf{Data Type} & Vectors / Tensors & Intensional Entities & Logic Formulas \\
  \hline
  \textbf{Role} &Perceives the world & Maps neurons to symbols & Reasons about perceptions \\
  \hline
  \textbf{Self-awareness} & Unconscious & The "\textbf{I}" links the two & Conscious autoepistemic \\
  & & & thought \\
  \hline
\end{tabular}
\\\\
In this paper we introduce more expressive 4-valued $IFOL_B$, based on Belnap's bilattice, by probabilistic features as well, in order to make neuro-symbolic learning and many-valued deductions (as humans) also in the presence of unknown and inconsistent (contradictory) sentences \cite{Majk25,Majk26}, because is such cases the standard 2-valued FOL is unable to work well (deduces absolutely anything). In this paper we will extend this truth-many-valued neuro-symbolic framework also to the probabilistic learning inside this $IFOL_B$ useful also for the simulation, planing and verification of different strategies in order that the AGI-robot's would be able to achieve their objectives and adaptive goal-directed behaviour under varying circumstances.
 With this extension, the robots will be able to reason both with algebraic many-valued method (as epistemic logic) and also to manage uncertainty by \emph{probabilistic logic} and symbol-guided \emph{common-sense} obtained by deep neural learning in the same formal $IFOL_B$ cognitive model. Such AGI-robots cognitive model is  a general highly unified neuro-symbolic model.

 It is passed twenty years from my  proposed a fairly ambitious revision in that time of the earlier probabilistic logic programming (PLP) frameworks developed by V. S. Subrahmanian and collaborators \cite{Majk07,Majk11}, where I tried to address exactly the kinds of their semantic weaknesses: problematic fixpoint semantics, interval ambiguity, mismatch between syntax and possible-world semantics, temporal inconsistency, and
many-valued logical complications. In particular, I criticized implicit temporal semantics, unclear Herbrand interpretation structure, interval annotations without explicit world semantics, overly syntactic fixpoint constructions. I argued that time should be explicitly represented, probabilistic worlds should remain classical possible worlds, probability should be defined over sets of temporal Herbrand interpretations. This is actually a fairly deep semantic correction which effectively reduces temporal probabilistic logic to classical logic over expanded temporal predicates with classical model theory,    accepted by the next developments of this revision, especially in today's main-stream of distributive semantics for PLP (modern PLP research mostly evolved toward distribution semantics, weighted logic, probabilistic graphical models, and differentiable probabilistic inference).

Earlier systems often treated time as metadata, annotation, external modal structure, while in my revision I instead internalized time directly into   Herbrand models, predicates and possible worlds, by restoring semantic clarity:
each possible world becomes an ordinary temporal interpretation, and
probability distributions are defined over classical logical models. Rather than relying primarily on
interval probability propagation and many-valued fixpoint operators, this revision has been philosophically important: probability should remain measure-theoretic and logic should remain 2-valued at the world level, which aligns more closely with
standard probability theory, modal semantics, Kripke semantics, and probabilistic model theory. I explicitly criticized the idea that probabilistic truth should be treated as ordinary many-valued logic, by adopting Nils Nilsson probability distributions over possible worlds, instead of embedding probabilities directly as many-valued truth values: The probabilities are not truth values themselves,
they are properties of propositions across worlds.  This direction is actually philosophically closer to Nilsson, Sato, and important researchers: Fabrizio Riguzzi, Luc De Raedt, Angelika Kimmig and Terrance Swift, with key systems ProbLog, PRISM and LPADs.   A major semantic consolidation paper was by Riguzzi $\&$ Swift on well-defined distribution semantics \cite{RiSw11}.

My revision reflected a deeper foundational issue: Is probability a generalized truth value, or a measure over classical worlds? This is a profound distinction, with the answer: worlds remain classical and probability is meta-level structure over worlds. That position is philosophically closer to Kolmogorov probability,
modal realism and stochastic semantics,
than, fir example, to fuzzy logic. This resembles modal logic, intensional semantics \cite{Majk11a,Majk11TS} with possible-world semantics and Montague-style semantics. Here I will extend it \cite{Majk22} also to 4-valued (Belnap's bilattice) logic for AGI robots.
\section{Nilsson's Structures and Reasoning about Probabilities\label{section:problogic0}}
The probability theory is a well-studied branch of mathematics, in
order to carry out formal reasoning about probability. Thus, it is
important to have a logic, both for computation of probabilities and
for reasoning about probabilities, with a well-defined syntax and
semantics. Both current approaches, based on Nilsson's probability
structures/logics, and on linear inequalities in order to reason
about probabilities, have some weak points (Section 2.2 in \cite{Majk22}).

We will show  that the logic for
reasoning \emph{about} probabilities can be naturally embedded into
a 4-valued intensional FOL with intensional abstraction, by avoiding
current ad-hoc system composed of \emph{two different} 2-valued
logics: one for the classical propositional logic at lower-level,
and a new one at higher-level for  probabilistic constraints with
probabilistic variables.

The main motivation for an introduction of the intensionality in the
probabilistic-theory of the propositional logic is based on the
desire to have the \emph{full} logical embedding  of the probability
into the First-Order Logic (FOL), with a clear difference from the
classic concept of truth of the logic formulae and the concept of
their probabilities. In this way we are able to replace the ad-hoc
syntax and semantics, used in current practice for Probabilistic
Logic Programs \cite{DeDe04,UZVS06,Majk07} and probabilistic deduction \cite{Majk09i}, by the standard syntax and semantics used for the FOL where the probabilistic-theory properties are expressed simply by the particular constraints on their interpretations and models.

In this section we will consider the probabilistic semantics for the
propositional logic only (it can be easily extended to predicate logics
 as well) \cite{Nils86,FaHM90} with a fixed finite set $P = \{p_1,...,p_n\}$ of primitive propositions, which can be thought of as corresponding to basic probabilistic events. The set $\L(P)$ of the propositional formulae is the closure of $P$ under the Boolean operations for conjunction and negation,
$\wedge$ and $\neg$, that is, it is the set of all formulae of the
\emph{propositional} logic $(P,\{\wedge, \neg\})$.

In order to give the probabilistic semantics to such formulae, we
first need to review briefly the probability theory (see, for example, \cite{Fell57,Halm50}):
\begin{definition} \label{Def:p-space}
A probability space $(S,\mathcal{X}, \mu)$ consists of a set $S$,
called the sample space, a $\sigma$-algebra $\mathcal{X}$ of subsets
of $S$ (i.e., a set of subsets of $S$ containing $S$ and closed
under complementation and countable union, but not necessarily
consisting of all subsets of $S$) whose elements are called
measurable sets, and a probability measure $\mu:\mathcal{X}
\rightarrow [0,1]$ where  $[0,1]$ is the closed
interval of reals from 0 to 1. This mapping satisfies
Kolmogorov axioms \cite{Kolm86}: \index{Kolmogorov axioms}

A.1 $~\mu(Y) \geq 0$ for all  $Y \in \mathcal{X}$.

A.2 $~\mu(S) = 1$.

A.3 $~\mu(\bigcup_{i\geq 1} Y_i) = \sum_{i\geq 1} \mu(Y_i)$, \\$~~$if
$~Y_i$'s are nonempty pairwise disjoint members of $~\mathcal{X}$. We define a probability density function, $KI= \mu\circ in$, where $in:S \hookrightarrow \P(S)$ is an inclusion such that $in(s) = \{s\}$.
\end{definition}
The  $\mu(\{s\}) = KI(s)$ is the value of probability in a single point of
space $s$.

 The property A.3 is called \emph{countable additivity} for
the probabilities in a space $S$. In the case when $~\mathcal{X}$ is
finite set, then we can simplify property A.3 above to

A.3' $~\mu(Z \bigcup Y) = \mu(Z) + \mu(Y)$, \\ if $Z$ and $Y$ are
disjoint members of $~\mathcal{X}$, or, equivalently, to the following axiom:

A.3'' $~\mu(Z ) = \mu(Z\bigcap Y) + \mu(Z \bigcap \overline{Y})$,\\
where $\overline{Y}$ is the compliment of $Y$ in $S$, so that $\mu(
\overline{Y}) = 1 - \mu(Y)$.

In what follows we will consider only finite sample space $S$, so
that $~\mathcal{X} = \P(S)$. Thus, in our case of a finite set $S$ we obtain, form A.1 and A.2, that for any $Y \in \P(S)$,
 \begin{equation}\label{eq:muKI}
 ~\mu(Y) = \sum_{s \in Y}\mu(\{s\}) = \sum_{s \in Y} KI(s)
\end{equation}
 Based on the work of Nilsson in \cite{Nils86} we can define for a given propositional logic with a finite set of primitive propositions $P$ the sample space $S = \textbf{2}^{P}$, where
 \begin{equation}\label{eq:muKIbb}
 \textbf{2} = \{f,t\} \subset X = \B_4 = \{f,t,\bot,\top \}
 \end{equation}
  is the set of truth values of standard 2-valued logic,
 so that the probability space is equal to the Nilsson's structure $N =
(\textbf{2}^{P}, \P(\textbf{2}^{P}), \mu)$.

In his work (page 2, line 4-6 in \cite{Nils86}) Nilsson considered a
Probabilistic Logic "in which the truth values of sentences can
range between $0$ and $1$. The truth value of a sentence in
\emph{probabilistic logic} is taken to be the probability of that
sentence in ordinary first-order logic." That is, he considered this
logic as a kind of a \emph{many-valued} (shown in Section 2.2.1 in \cite{Majk22}), but not a compositional
truth-valued, logic. But in his paper he did not defined the formal
syntax and semantics for such a probabilistic logic, but only the
matrix equations where the probability of a sentence $\phi \in
\L(P)$ is the sum of the probabilities of the sets of possible
worlds (equal to the set $S = \textbf{2}^{P}$) in which that
sentence is \emph{true}. So that he assigns \emph{two} different
logic values to each sentence $\phi$: one is its probability value
and another is a classic 2-valued truth value in a given possible
world $v \in \W = S = \textbf{2}^{P}$.

The \emph{logic} inadequacy of this seminal work \cite{Nils86} of
Nilsson is also considered in \cite{FaHa89}, by extending this
Nilsson's structure into a more general  \emph{probability structure} $M =
(\textbf{2}^{P}, \P(\textbf{2}^{P}), \mu, \pi)$, where $\pi$
associates with each $s \in S = \textbf{2}^{P}$ the truth
assignment $\pi(s):P \rightarrow \textbf{2}$. However, in our case when $S = \textbf{2}^{P}$, $\pi$ is just an identity, so not necessary, and we consider each $s$ as a truth valuation $s=v:P\rightarrow \{f,t\}$ which can be uniquely
extended to the truth assignment $v^*$ to all formulae in $\L(P)$, by
taking the usual rules of propositional logic (the unique
homomorphic extension to all formulae), and we can associate to each
propositional formula $\phi \in \L(P)$ the set $\phi^M$
consisting of all states $s \in S$ where the sentence $\phi$ is true, so that
\begin{equation}\label{eq:muKIaa}
~\|\phi\| = \{v \in \W = \textbf{2}^{P} ~|~ v^*(\phi) =t\}~~
\end{equation}
But, differently from Nilsson, in \cite{FaHa89} the authors did not define a many-valued
propositional logic, but a kind of 2-valued logic based on
probabilistic constraints. They denoted by $w_N(\phi)$ the
\emph{weight} or \emph{probability} of $\phi$ in Nilsson structure
$N$, correspondent to the value $\mu(\|\phi\|)$, so that the basic
probabilistic 2-valued constraint can be defined by expressions $c_1
\leq w_N(\phi)$ and $w_N(\phi) \leq c_2$ for given constants
$c_1,c_2 \in [0,1]$. They expected  their logic to be used for reasoning \emph{about} probabilities. But, again,  they did not defined a
unique logic, but \emph{two different} logics: one for the classical
propositional logic $\L(P)$, and a new one for 2-valued
probabilistic constraints obtained from the basic probabilistic
formulae above and Boolean operators $\wedge$ and $\neg$. They did
not consider the introduced symbol $w_N$ as a formal functional
symbol for a mapping $w_N:\L(P) \rightarrow [0,1]$, such that for
any propositional formula $\phi \in \L(P)$, with $\W = S = \textbf{2}^P$, the \emph{probability to be true} of sentence $\phi$ is

$w_N(\phi) = \mu(\|\phi\|) = \sum_{v\in \|\phi\|} KI(v)$. \\
Instead of this intuitive meaning for $w_N$ they considered each expression $w_N(\phi)$ as a particular probabilistic term (more precisely, as a structured
probabilistic \emph{variable} over the domain of values in $[0,1]$).

 It seams that such a dichotomy and difficulty to have
\emph{a unique} 2-valued probabilistic logic, both for an original
propositional formulae in $\L(P)$ and for the probabilistic
constraints, is based on the fact that if we consider $w_N$ as a
function with one argument then it has to be formally represented as
a binary predicate $w_N(\phi,a)$ (for the graph of this function)
where the first argument is a formula and the second is its
resulting probability value. Consequently, a constraint "the
probability of $\phi$ to be true is less or equal to $c$", has to be formally
expressed by the logic formula $w_N(\phi,a) \wedge \leq (a,c)$ (here
we use a symbol $\leq$ as a built-in rigid binary predicate where
$\leq (a,c)$ is equivalent to $a \leq c$), which is a
\emph{second-order} syntax because $\phi$ is a logic \emph{formula}
in such a unified logic language. That is, the problem of obtaining
the unique logical framework for probabilistic logic comes out with
the necessity of a \emph{reification} feature of this logic
language,
analogously to the case of the \index{reification}
intensional semantics for RDF data structures \cite{Majk08ird,Majk22}.

Consequently, we need a logic \cite{Majk11a} which is able to deal directly with
reification of logic formulae,  that transforms a
propositional formulae $\phi \in \L(P)$ into an abstracted
\emph{term}, denoted by $\lessdot\phi \gtrdot$. By this approach the expression $w_N(\lessdot\phi \gtrdot,a) \wedge \leq (a,c)$ remains
to be an ordinary first-order formula. In fact, if $\lessdot\phi
\gtrdot$  is translated into non-sentence "that $\phi$", then the
first-order formula above corresponds to the sentence "the
probability \verb"that" $\phi$ is true is less than or equal to $c$".

Such approach has been used \cite{Majk05E,Majk22} for the reduction of temporal probabilistic databases into constraint logic programs and to apply the interval PSAT in order to find the models of such interval-based probabilistic \emph{programs}. In our case, for selfindependent neuro-symbolic AGI robots, we do not need to write the probabilistic \emph{programs} for them, but only to provide a general self-reasoning about the probabilities of their 4-valued sentences. So,
in next Section we will show how the probabilistic reasoning can be embedded in the 4-valued $IFOL_B$ based on Belnap's bilattice of truth-values in $X = \B_4 = \{f,t,\bot,\top\}$.
\section{Neuro-symbolic AGI Robots  Reasoning about Probabilities of Uncertain Events\label{section:problogic}}
There are numerous proposals for probabilistic logics. Very roughly, they can be categorized into two different classes: those logics that attempt to make a probabilistic extension to logical entailment, such as Markov logic networks, and those that attempt to address the problems of uncertainty and lack of evidence (evidentiary logics).

 For AGI  robots we considered the  $IFOL_B$ with 4-valued Belnap's bilattice of truth-values in $X = \B_4 = \{f,t,\bot,\top\}$ with knowledge ordering as well \cite{Majk25}, where the value "unknown" is the bottom value, the sentences with this value are indeed unknown facts, that is, the missed knowledge in the AGI robots.

  Thus, these unknown facts are not part of the robot's knowledge database, and by learning through input and experiences, the robot's knowledge would be naturally expanded over time.
  Consequently, this phenomena has been represented by the  Closed Knowledge Assumption and Logic Inference \cite{Majk26}, here for unknown facts, because of lack of evidence inside non-probabilistic $IFOL_B$ and its autoepistemic deductive system enable to establish some of the \emph{known} logic states in $\{t,f,\top\} \subset X$ of the facts (to be true, false or inconsistent) the robot has to assign to such facts the "absolute uncertainty", that is the truth-value $\bot$ (\emph{unknown} logic value). So, the minimization of this uncertainty can be obtained by introducing the probabilistic computation (by using Nilsson's structures with Kolmogorov's axioms) for these unknown facts, to obtain \emph{what is the probability} of a currently unknown fact to be  true, false or inconsistent.

We recall that the  extensionalization functions $h \in \E_{in}$ in $IFOL_B$, are given in the disjoint union from (\ref{eq:dueM6}),
 $$h =   \sum_{i\in \mathbb{N}}h_i:\D \longrightarrow   D_0+\sum_{i\geq 1}\mathfrak{Rm}_i$$
 Thus,
the intensions can be seen as \emph{names} of abstract or concrete
entities, while the extensions correspond to various rules that
these entities play in different worlds
 we use  the bijective mapping (\ref{eq:SetExtens3}), $is_{in}:\W \rightarrow \E_{in}$, where $\E_{in}$ is the set of all extensionalization functions that respect all  built-in predicates, and the "set of possible worlds" $\W = \{h\circ I | h \in \E_{in}\}$.
\\
So,   from the commutative truth-diagram for the set $\L_0$ of all sentences of this logic $\L_{in}$, provided by Theorem 1 in \cite{Majk25},  we obtain that $I^*_{B} = h\circ I$,
so that from Definition 15 in \cite{Majk25},  given an assignment $g:\V\rightarrow \D$, for any sentence $\phi/g\in \L_0$,
 \begin{equation}\label{eq:mvR}
 (h\circ I)(\phi/g) =I^*_{B}(\phi/g) =
 \left\{
    \begin{array}{ll}
 \{v^*(\phi/g)\} \in  \mathfrak{Rm}, & \hbox{if $~~v^*(\phi/g) \neq \bot$}\\
      \emptyset, & \hbox{otherwise}
    \end{array}
  \right.
 \end{equation}
 Consequently, we are able to represent the whole \emph{AGI robot's knowledge Database} of its n-ary concepts, $n \geq 1$, in each fixed instance of time (it is an analog to traditional relational Database with standard 2-valued FOL) as follows:
\begin{definition}\label{Def:knowledge} \textsc{AGI Robot's Current Atomic Knowledge Database}\footnote{This definition is different from the definition in \cite{Majk25}, in order to be able also to derive the current Herbrand model of robot's knowledge $v:H \rightarrow X$, where $H$ is the Herbrand base for all  predicates in $P$, Note that the autoepistemic predicate $Know$ is a meta-èredicate, so that $Know \notin P$.}:\\
For a given instance of time, the current AGI robot knowledge is defined by its current MV-interpretation $\textbf{I}^*_{B}$ (that represents robot's current world $w \in \W$) defines  the extension of robot's knowledge Database $\K$ by:
\begin{equation} \label{eq:two-know}
~~ \K ~= ~ \{R_{p_i^k} = \textbf{I}^*_{B}(p_i^k(x_1,...,x_k) ~|~p_i^k \in P, ~for~ x_i \in \V,  1\leq i\leq k\}
\end{equation}
where $\textbf{I}^*_{B}$ is the current  MV-interpretation (a function from $\L$ to $\mathfrak{Rm}$) and $\V$ the set of variables and $R_{p_i^k}$ is the $(k+1)$-ary relation obtained for the k-ary predicate $p_i^k$.
\end{definition}
 Consequently, the atomic knowledge database is the subset of current metaknowledge of a robot is the set of relations, that is,
\begin{equation} \label{eq:two-metaknow}
\K \subset~~  ~ \{R ~|~R  \in Im(\textbf{I}^*_{B}), ~for~ ar(R) \geq 2\}
\end{equation}
where for each tuple of ground terms, $\textbf{d} = (t_1,...,t_k,a) \in R_{p_i^k} \in \K$, of relation with arity $k+1\geq 2$, the current truth-value of this, for robot known, fact is equal to last value of this tuple, $a =\pi_{k+1}(\textbf{d}) \in \{f, \top, t\}$ with truth-ordering $f < \top < t$. So, for any ground instance of each robot's (real or virtual) predicate (and corresponding intensional concept), robot is able to know the level of truth in a given instance of time. Thus, given current atomic knowledge database $\K$, we are able to derive from it the current Herbrand model $v:H \rightarrow X$ and its extension $v^*$ to all sentences, for robot's knowledge by, for any ground atom $p_i^k(t_1,...,t_k)\in H$, with the ground terms $t_i$, for $1\leq i\leq k$,
\begin{equation} \label{eq:two-metaknowH}
 v(p_i^k(t_1,...,t_k)) =
 \left\{
    \begin{array}{ll}
 a\in \{f, \top, t\}, & \hbox{if $~~(t_1,...,t_k,a) \in R_{p_i^k} \in \K$}\\
      \bot, & \hbox{otherwise}
    \end{array}
  \right.
\end{equation}
Thus, from current atomic knowledge database $\K$, we are able to derive current Herbrand model $v$, its extension to all sentences $v^*$ and hence the current possible world (MV-model in Definition \ref{Def:knowledge}) $\textbf{I}^*_B = h\circ I \in \W$.

 It was demonstrated that the knowledge database $\K$ satisfies the following assumption:
 \begin{definition}\label{Def:CKA} \textsc{Closed Knowledge Assumption(CKA)}:\\
The CKA for the  many-sorted Intensional FOL based on Belnap's 4-valued bilattice of truth-values is defined, for every Herbrand interpretation $v:H \rightarrow X$ and assignment $g:\V \rightarrow \D$, as follows:\\
For each k-ary virtual predicate $\phi(x_1,...,x_k)$, $k \geq 1$,

$v^*(\phi(x_1,...,x_k)/g) \neq \bot~~$ iff $~~(g(x_1),...,g(x_k))\in \pi_{-i}(I_B^*(\phi(x_1,...,x_k)))$.
\end{definition}
The \emph{intensional abstract terms} are "that-clauses" used for \emph{reification} \index{reification} features \cite{Majk09P} of $IFOL_B$ so that, for a logic formula $\phi(\textbf{x})$ and assignment $g\in \D^\V$ of variables in $\V$,  "that $\phi$", from Definition \ref{def:abstrConv} is  denoted by the ground abstracted term $\lessdot \phi(\textbf{x})[\beta/g(\beta)] \gtrdot_\alpha$ (if $\phi$ is a sentence then both $\alpha$ and $\beta$ are empty).  Hence, the sentence "the probability that a sentence $\phi$ has a truth-value $a$ is less then or equal to $c_1$" can be expressed by the first-order logic ground atom $w_N(\lessdot \phi \gtrdot,a,c) \wedge \leq(c,c_1)$, where  $\leq$ is the binary built-in predicate with standard denotation $c\leq c_1$, while "the probability that $\phi$ has a truth-value $a$ is equal to $c$" is denoted by the  ground atom $w_N(\lessdot \phi \gtrdot,a,c)$ with the special  ternary  predicate  $w_N$ which first argument is from an sentence abstracted term, the second argument is for the truth-value of this sentence and third argument is the probability that it is so.
 So, we introduce the following  built-in concepts (for Definition 9 in \cite{Majk25}) of typed (many-sorted) $IFOL_B$:
 \begin{enumerate}
   \item "\textbf{truth values}:$~s_X$", introduced in \cite{Majk25}, that is a concepts in $D_2$ such that for $X \subset D_0$, its extension is $h_2(\textbf{truth values}:~s_X) = \{(a,t)~|~a,t \in X\} $.  We will use the special typed variable $x_X$ for the set of truth values in $X = \B_4 = \{f,t,\bot,\top\}$, so that for each assignment $g$, $g(x_X) \in X$.
   \item "\textbf{probability}:$~s_p$", that is a concepts in $D_2$ such that  its extension is $h_2(\textbf{probability}:~s_p) = [0,1]\times \{t\}$, where $[0,1]$ is interval of positive reals used for the probabilities and $t\in X$ the true logic value. We will use the special typed variable $x_p$ for the probabilities, so that for each assignment $g$, $g(x_p) \in [0,1]$.
  \item for each finite ordered calendar interval of time $[\textbf{t}_I,\textbf{t}_F]_\tau$  with initial $\textbf{t}_I$ time-instance and final $\textbf{t}_F$ time-instance, of sorts $s_\tau$) with a given granularity $\tau \in \{year, year:month, year:month:day, year:month:day:our,...\}$,
      we can have a particular calendar-concept in $D_2$,  "\textbf{calendar}:~$s_\tau$",  such that its finite extension is       $h_2(\textbf{calendar}:~s_\tau) = [\textbf{t}_I,\textbf{t}_F]_\tau \times \{t\}$  with $t\in X$ the true logic value. We will use the special typed variables $x_\tau$ for the time-instances of these calendars, so that for each assignment $g$, $\textbf{t} = g(x_\tau) \in [\textbf{t}_I,\textbf{t}_F]_\tau$.\\
      Each predicate $p_i \in P$ that has exactly one attribute of sort  $s_\tau$ is called an "atomic event" predicate.
 \end{enumerate}
 Thus, in this version of $IFOL_B$ we will have the following three meta-predicates (two-valued predicates that represent's the semantically very specific knowledge about properties of standard domain-based predicates in $P$ (the predicates in $P$ have no the typed variable $x_s$ of sort 'nested sentence', provided in Definition of static sorts in \cite{Majk24ar}):
 \begin{itemize}
   \item 4-valued Knowledge predicate $Know$ used for autoepistemic deduction \cite{Majk26}.
   \item 2-valued (true or false) Probabilistic predicate $w_N$  with the first argument has the sort of 'nested sentence' and with free-variable atom $w_N(\lessdot \psi(\textbf{x})\gtrdot^\beta, x_X, x_p)$ where the formula $\psi$ is composed by predicates in $P$. If $\psi$ is composed also by "atomic event" predicates, then $w_N(\lessdot \psi(\textbf{x})\gtrdot^\beta, x_X, x_p)$ is called Event probabilistic predicate".
  \end{itemize}
 so that $P \bigcap \{Know, w_N\}$ is empty set. This separation of standard and meta-predicates is based on the fact that the truth-value  of meta-predicate ground atoms depends on the truth-values of the \emph{another sentence} used in the abstract term of the typed variable $x_s$.
 In effect, from Definition 14  in \cite{Majk26}, we have that the computation of the truth-value of the 4-valued $Know$ ground atoms is done as this atom is a particular logic sentence (formula) and not ground atom of Herbrand base:
 \begin{multline} \label{eq:esem3s}
 v^*(Know(t_1,t_2,\lessdot \psi(\textbf{x})\gtrdot^\beta)/g ) = v^*(\psi(\textbf{x})/g)\in X, ~~~~and~ for ~~I_B^* = is_{MV}(v^*)\\
  I_B^*(Know(t_1,t_2,\lessdot \psi(\textbf{x})\gtrdot^\beta)/g ) = I_B^*(\psi(\textbf{x})/g)~~~~~~~~~~~~~~~~~~~~~~~~~~~~~~~~~~~~~~~~~~~~~~~~~~~~~~~
 \end{multline}
 We have shown by Lemma 1 in \cite{Majk26} that, during time-evolution for time-instances $\textbf{t}_i \leq \textbf{t}_{i+1}$ of robot's knowledge, it is always satisfied that
 \begin{equation} \label{eq:knowledge}
 v^*_i\preccurlyeq_k v^*_{i+1}
\end{equation}
 that is, we have a monotonic increments of knowledge, initially most (non built-in predicate) ground atoms are unknown, and that during learning processes in future, the (non built-in predicate) ground atoms can change their truth-value in this knowledge-monotonic way:
 \\
1. $\bot \mapsto f$  of $\bot \mapsto t$;\\
2. $f \mapsto \top$;\\
3. $t \mapsto \top$;\\
4. $\top \mapsto \top$.\\
Thus, for the Herbrand base (from Definition 4 in \cite{Majk26})

$~~H = \{p_i^k(t_1,..,t_k)~|~p_i^k \in P$ and $t_1,...,t_k$ are ground terms $\}$\\ and the current Herbrand interpretation $\textbf{v}:H\rightarrow X$,  \emph{for  still unknown atoms} $A\in H$ with $\textbf{v}(A) = \bot$ (note that such atoms can note be of built-in predicates that have invariant truth-value true or false) we can use the probabilistic methods to establish with which probability they can have any truth-value in $X$, and to use this knowledge for another logical deductions.

 Consequently, in order that AGI robots be able to reason about probabilities of the  sentences  using the syntax of the FOL
with the set of predicate symbols $P$ of current robot's knowledge (robot would be able to increment this set with new predicates as well, based on its extended in time knowledge about external world) we have the sample space of Definition \ref{Def:p-space}, from (\ref{eq:Herbrand2}),  $S \subseteq \I_H \subset X^{H}$ where $H$ is Herbrand base of this logic $IFOL_B$  with Belnap's bilattice of truth-values in $X$.
 And we impose the following constraints for the probability density $KI = \mu\circ in:S \rightarrow [0,1]$:
\begin{definition}\label{def:KIconstraint} \textsc{Probability Density Constraints}:\\
For any current Herbrand interpretation $\textbf{v}:H\rightarrow X$, the robot's current knowledge database $\K$ of $IFOL_B$ (from Definition \ref{Def:knowledge} )  must \emph{remain invariant} by introducing the probability structure $(S,\P(S),\mu)$, in Definition \ref{Def:p-space}.\\
That is, for every atom $A \in H$ such that $u = \textbf{v}(A) \neq \bot$, the probability that $A$ has the truth-value $u$ must be equal to 1, i.e.,
\begin{equation}\label{eq:KIconstraint}
\forall A\in H.(~if ~~u = \textbf{v}(A) \neq \bot ~~~~then ~~\sum_{v\in S, v(A) = u} KI(v) = 1)
\end{equation}
\end{definition}
Thus, the action of introducing the probability computation in $IFOL_B$ will have the effects only to diminish the uncertainty of unknown facts, for which we can now compute \emph{the probability} to have any truth-value in $X$. \\
\textbf{Remark}: This is a particular case of Noether's conservation theorem applied to robot's cognitive system: the (global) \emph{symmetry transformation} of robot's cognitive power by action of introduction of the probability structure $(S,\P(S),\mu)$ preserves its knowledge database $\K$. \\
This symmetry is \emph{global} because is valid for all current Herbrand base (that is to all predicates in $P$ of robot's knowledge base: so computationally it has relevant cost to be applied in real-time robot's processes. However, we are able to consider less computationally expensive symmetry transformation, that can be done dynamically by robot to resolve some decisions about concrete problems that involve a very restricted subset of it predicates, relevant to such problems: such dynamic symmetries will be called a \emph{local symmetries} in what follows.
\\$\square$
\\
So, in what follows we will extend the  globally consistent and empirically satisfactory unification of classic probability theory and standard (two-valued) first-order logic that is suitable for inductive reasoning, developed  by  Gaifman and Snir \cite{Gaif64,GaSn82} , to our more sophisticated Belnap's based 4-valued typed $IFOL_B$ used for humanoid AGI robots.
\begin{definition}\label{def:mvR} \textsc{Global  Current Nilsson's Structure for $IFOL_B$}:\\
Let $\textbf{v}:H\rightarrow X$ be the current Herbrand interpretation of AGI robot with Belnap's based 4-valued typed $IFOL_B$, and corresponding current possible world $\textbf{I}_B^* = is_H(\textbf{v}) \in \W$.  Then we define the current Nilsson's probability structure $(S,\P(S),\mu)$ satisfying Kolmogorov axioms in Definition \ref{Def:p-space}, such that the sample space $S$ is defined as a subset of $\I_H$ in (\ref{eq:Herbrand2}), by
\begin{equation}\label{eq:IFOLpstructure}
S =_{def} \{v \in \I_H \subset X^H | ~~for~ every ~~A \in H, v(A) = \textbf{v}(A) ~~if~ \textbf{v}(A) \neq \bot\}
\end{equation}
so that for any $v\in S$, and ground atom $A\in H$ such that $\textbf{v}(A) = \bot$, we can have that $v(A)\in X$ can have the value different from $\bot$ as well.
\end{definition}
Note that this sample space $S$ in (\ref{eq:IFOLpstructure}) is defined only for current possible world (and current knowledge base $\K$, and does not modify the logic truth-value of ground atoms $A\in H$ for which $\textbf{v}(A) = \bot$ or currently unknown sentences $\phi$ with    $\textbf{v}^*(\phi) = \bot$, but only offers the capacity to compute their \emph{probability} to have any truth value in $X$.\\
This global current Nilsson's probability structure has the following important properties:
\begin{coro}\label{coro:IFOLpstructure}
The global current Nilsson's probability structure with sample set $S$ in (\ref{eq:IFOLpstructure}) satisfy the probability density constraints in Definition \ref{def:KIconstraint}.\\
 Thus, it is consistent with current robot's knowledge database $\K$ (in Definition \ref{Def:knowledge}), consistent with Kolmogorow probability axioms and logical deduction, and allows inductive reasoning and confirmation of universally quantified hypotheses.
\end{coro}
\textbf{Proof}: It is enough to show the consistency with current robot's knowledge database $\K$. In fact, for each ground atom $A \in H$ for which the current truth-value in $\K$ is $u = \textbf{v}(A) \neq \bot$, we have that  the probability density constraint (\ref{eq:KIconstraint}) is satisfied. That is, the probability $p_u$ that this atom has the truth-value $u$ is equal to 1:

$p_u = \sum_{v\in S, v(A) = u} KI(v) $

$= \sum_{v\in S} KI(v)$ ~~~~~~~~~~~~~~from (\ref{eq:IFOLpstructure})

$ = \mu(S)$

$ = 1$.\\
which preserves robot's current knowledge database and all logical deductions from them.
\\$\square$\\
Consequently, for  new predicate $w_N$ (not in $P$), we obtain that for any assignment $g$ and formula $\psi(\textbf{x})$ composed by predicates (only) in $P$, if, for a given assignment $g$, in current world $\textbf{v}^*(\psi(\textbf{x})/g) = \bot$, then we can compute in this current world (current knowledge database $\K$)
 \begin{multline} \label{eq:esem3s2}
 \textbf{v}^*(w_N(\lessdot \psi(\textbf{x})\gtrdot^\beta/g, g(x_X),g(x_p) )) =\\ 
 =\left\{
    \begin{array}{ll}
 t, & \hbox{if $~~
 ~g(x_p) = \sum_{v\in S, v^*(\psi(\textbf{x})/g) =g(x_X)} KI(v) $}\\
      f, & \hbox{otherwise}
    \end{array}
  \right. ~~~~~~~~~~~~~~~~~~~~~
  \end{multline}
That is,
\begin{enumerate}
  \item The $w_N(\lessdot \psi(\textbf{x})\gtrdot^\beta/g, g(x_X),g(x_p))$ is true iff \emph{the probability} that $\psi(\textbf{x})/g$ has the truth-value $g(x_X)$ is equal to  $g(x_p)$.
  \item In the case when the formula $\psi$ contains also "atomic event" predicates, the event $w_N(\lessdot \psi(\textbf{x})\gtrdot^\beta/g, g(x_X),g(x_p))$, with $\tau \in \beta$ is a free variable in $\psi$, is true iff the probability that $\psi(\textbf{x})/g$ has the truth-value $g(x_X)$  at time-instance $g(x_\tau)$ is equal to  $g(x_p)$.  Note that from the fact that $x_\tau$ is typed variable, the instance of time $g(x_\tau)$ is always inside the finite interval of granular calendar specified by this type $\tau$. Obviously, $\psi$ can be composed by a number of different "atomic events" as well, each of them with particular granular calendars time-variables $x_\tau$ and their intervals.
\end{enumerate}
\textbf{Remark:}
Thus, for the current possible world of robot (its current knowledge database $\K$), for any unknown fact (which is not in $\K$), we can have additional knowledge of what is the probability that this fact is true, false or inconsistent. In this way, by probabilistic computation based on Nilsson's probability stricture in Definition \ref{def:mvR} we diminish generally the robot's knowledge uncertainty by preserving (from Corollary \ref{coro:IFOLpstructure}) its consistency and logical deduction.\\
That is, for any sentence $\phi \in \L_0$ such that in current Herbrand interpretation $\textbf{v}:H\rightarrow X$, $\textbf{v}^*(\phi) = \bot$, robot would be able to compute the probability $p_u$ that this sentence has the truth value $u \in X$,
\begin{equation} \label{eq:prob}
p_u = \sum_{v\in S, v^*(\phi) = u} KI(v)
\end{equation}
 which the robot can use for probability-based decisions and actions, by using the reasoning about the probabilities (as in Section 2.2.2 in \cite{Majk22}) by using the true atoms $w_N(\lessdot \phi\gtrdot, u, p_u)$ with also "atomic events"   
 that tells to robot that "the probability at time-instance $\textbf{t}_i$ that $\phi$ has the truth-value $u$ is equal to $p_u$".

  These true facts about the probabilities of the unknown sentences to have a particular truth-value can be also inserted in the conscious part of robot's cognitive system (see (\ref{eq:esem2}) for more details) by,

  $Know(in ~presence, \textbf{I}, w_N(\lessdot \phi\gtrdot, u,p_u))$,\\
 and hence to be used by robot's autoepistemic deductions and explanations as well.
 \\$\square$
\begin{example} \label{ex:events}
For example, let us consider an "atomic event" predicate $p_i$, with its atom of free variables $p_i(...,x_\tau,...)$, such that there exists the time-instance $\textbf{t} \in [\textbf{t}_I,\textbf{t}_F]_\tau$ and assignment $g$ such that $g(x_\tau) = \textbf{t}$, and in current possible world $\textbf{v}(p_i(...,x_\tau,...)/g) = \textbf{v}(p_i(...,\textbf{t},...)/g) = t$, which means that this event of ground atom $p_i(...,\textbf{t},...)/g$ will surely happen, and so we have also that $\textbf{v}^*((\exists x_\tau.p_i(...,x_\tau,...))/g) = t$.\\
Now, suppose that in this current world it is unknown if this event will happen, that is we have that $\textbf{v}^*((\exists x_\tau.p_i(...,x_\tau,...))/g) = \bot$ (it can happen only iff  for each $\textbf{t} \in [\textbf{t}_I,\textbf{t}_F]_\tau$, $\textbf{v}(p_i(...,\textbf{t},...)/g) = \bot$). So, in this currently completely uncertain event-state, a robot would be able able to compute the probability that it will happen, by using (\ref{eq:esem3s2}), so that for $x_\tau \notin \beta$ ($\beta$ is the set of only free variables),
\begin{multline} \label{eq:esem3s5}
 \textbf{v}^*(w_N(\lessdot \exists x_\tau.p_i(...,x_\tau,...))\gtrdot^\beta/g, t,g(x_p) )) =\\ 
 =\left\{
    \begin{array}{ll}
 t, & \hbox{if $~~
  ~g(x_p) = \sum_{v\in S, v^*(\exists x_\tau.p_i(...,x_\tau,...))/g = t} KI(v) $}\\
      f, & \hbox{otherwise}
    \end{array}
  \right. ~~~~~~~~~~~~~~~~~~~~~
  \end{multline}
  that is, $g(x_p)$ is the probability that this event will happen.
  \\$\square$
\end{example}
In this way, by using the current Nilsson's probability structure, we obtain the logical system of robots that are probabilistically-open (are not constrained by the pure-logical Closed Knowledge Assumption \cite{Majk26}) by \emph{minimizing the uncertainty of unknown sentences} in most prudent way.

This is obtained by the this neuro-symbolic AGI of robots based on symbolic $IFOL_B$ and this conservative probabilistic extension.
 The \emph{neural component} of robot's AGI cognitive system has to be used for computation of the probability density function $KI: S \rightarrow [0,1]$ based on the principle of maximum information entropy. The concept of information entropy was introduced by Claude Shannon  \cite{Shan48} and is also referred to as Shannon entropy or  \emph{Information entropy }, which is a mathematical measure of the average uncertainty, randomness, or "surprise" inherent in a set of data or events. Higher entropy means the outcome is less predictable, while lower entropy indicates a more certain and predictable outcome. In neuroscience and physics  this  concept has been adapted to model the flow of information in neural pathways and shares deep mathematical roots with thermodynamic entropy in physics.

  The principle of maximum entropy states that, among all probability distributions consistent with a given set of constraints (\ref{eq:KIconstraint}) in Definition \ref{def:KIconstraint}, the distribution that maximizes Shannon entropy should be selected. This yields \emph{the least committal distribution} (most prudent minimization of the uncertainty) compatible with the known constraints, introducing no structure beyond what is logically implied by the available information, which corresponds to the following general principle in physics:
 \\ \\
 \textbf{The Principle of Least Action} (committal distribution) KI: \\
  The \emph{least} committal distribution is the probability distribution KI that maximizes information entropy subject to whatever constraints you currently know about the data. In statistics and information theory, this is formally known as the Principle of Maximum Entropy. This principle minimizes the absolute uncertainty of any unknown fact in current \emph{internal model} $\textbf{v}:H\rightarrow X$, by giving the information about the probability that these facts are true, false or inconsistent.  Thee sense of this Principle of Least Action is that it does not modify the knowledge database $\K$ (that is, does not modify the robot's internal model $\textbf{v}$).\\
  We can consider it as the \emph{free energy principle} (of Karl Friston's Active Inference) which is a mathematical principle of information physics. Its application to robot's  "brain" reduces surprise or uncertainty  by making predictions based on internal models and uses "sensory input" (in our case the committal distribution KI) to extend its models so as to improve the accuracy of its predictions (in our case of unknown, and hence totally uncertain,  sentences).  In effect, with this Principle of Least Action, we apply a particular case of free energy principle of Karl Friston, which he  used in Bayesian approaches to human brain function and some approaches to artificial intelligence (introduced it as an explanation for embodied perception-action loops in neuroscience \cite{FrKH06}).\\
  More information is provided at the end, in Appendix, Section \ref{sec:Helmholtz}.
  \\$\square$\\
The justification is that entropy measures the expected information content (or log-surprise) of outcomes relative to a specified reference measure. Maximizing entropy ensures that no additional structure is imposed beyond the stated constraints. Principle of maximum entropy may be taken to compute degrees of belief of formulae \cite{GHKo94}, and it is shown
in \cite{PaVe90}  for the consistent probabilistic inference. This
method applied to probabilistic logic programming
\cite{LuKe99,Majk11}, based on conditional probabilistic clauses, has
shown that reduces the original entropy maximization
to relatively small optimization problems,
  This entropy of $KI$,  for current Nilsson's structure in Definition \ref{def:mvR}, is defined by,
\begin{equation}\label{eq:KIentropy}
H(KI) = - \sum_{v\in S}KI(v) \cdot logKI(v)
\end{equation}
 We can use robot's neural networks to calculate maximum-entropy probability distributions. Instead of solving complex mathematical equations, we use a dedicated robot's neural network as a numerical optimizer. The network learns to generate the correct probabilities by adjusting its weights to maximize entropy while respecting data constraints. This approach is particularly useful when dealing with variables that have many possible states (high dimensionality). The \emph{standard architecture} for this task involves a neural network in which the last layer applies a function called Softmax:
 \begin{itemize}
   \item Softmax Output: The Softmax function takes the numbers generated by the network and transforms them into a set of probabilities.
   This automatically solves the normalization constraint.
   \item Constraints: they are directly inserted into the rule with which the network learns.
   \item Loss Function: To train the network, we don't use classic classification errors. Instead, we create a custom loss function based on Lagrange multipliers.
 \end{itemize}
Moreover, the scoring function  can be also approximated by a nonlinear architecture, namely, a feedforward neural network (FFNN). Also know as deep neural network, artificial neural network or multilayer perceptron, it is a simple model where several linear combinations of inputs are passed through nonlinear activation functions called nodes. A set of nodes is called a layer, and the output of ones layer’s node becomes the input of the next layer in multilayer architectures. Finally, the output of the last layer is combined linearly to produce a cost-to-go \cite{Tadd19} and \cite{Bert20} (Section 4.2 about Multilayer and Deep Neural Networks).

\section{AGI Robots in Local Neuro-symbolic Dynamic Actions: an $IFOL_B$ Example}
In this section we will considered the \emph{dynamic real-time} probabilistic actions, baaed on \emph{local} symmetry transformations of robot's cognitive power, to resolve, as human do, the probabilistic decision on very focused (non-global) problems in real-time. Instead of time-expensive global symmetry transformation, described in previous section, that has the computational consequences over all  predicates of the whole robot's current knowledge, each practical problem to resolve does not require the involvement of all predicates (knowledge), but of only a very restricted subset of these predicates relevant to describe this problem by a particular sentence $\phi$. This is a robot's analogy with human brain localized neural activations (activating the "\emph{focused attention}") to resolve different cognitive functions, depending on the current cognitive problem by using the principle of \emph{least robot's knowledge action} (AGI analog to human brain free energy principle, Active Inference of Karl Friston).\\
\emph{What is being transformed} by global and local symmetry transformation is the belief state about unknown information not present in robot's knowledge database: with  these symmetry transformation a robot becomes \emph{cognitively generative model} by generation the set of probabilistic projections expressed by new predicate $w_N$ which is not a part of standard predicates in $P$ and hence does not modify Herbrand base $H$. However, these new probabilistic facts can be used for logical inference, for hypothetical simulations,  or decisions based on these probabilities between different concurrent plans. Under these symmetry transformations remains invariant the whole robot's knowledge base $\K$.

So, for example,
 decision process can be also divided in a number of subproblems, each of which described in relatively small number of predicates relevant (by focusing attention) to this subproblem.  In each of these subproblems our probability Nillson's space $S$ will be enormously smaller w.r.t. the global space provided by Definition \ref{def:mvR}, and hence the neural networks would be able to calculate the maximum-entropy probability distribution for such smaller probability structure in real-time, providing the dynamic robot's probabilistic decisions about these subproblems. \\
 Thus, we introduce this efficient local symmetry transformations, based on a (sub)problem-sentence $\phi$ describing this (sub)problem.
\begin{definition}\label{def:mvRLocal} \textsc{Local  Current Nilsson's Structures for $IFOL_B$}:\\
Let $\phi$ be a sentence, for which the robot has to take some decisions and consecutive actions, but for which the current truth-value is $\bot$ (unknown), so that is completely uncertain. Then we denote by $P_\phi$ the subset of predicates used in sentence $\phi$ ($P_\phi \subset P$) and by $H_\phi$ the Herbrand base (where $|H_\phi|$ denotes the finite number of atoms in $H_\phi$) of the small subset of predicates in $P_\phi$ of robot's complete cognitive system, so that $|H_\phi| << |H|$.\\
Let $\textbf{v}:H_\phi\rightarrow X$ be the restriction of current Herbrand interpretation to the atoms in $H_\phi \subset H$ so that $\textbf{v}^*(\phi) = \bot$ as well, and corresponding current possible world $\textbf{I}_B^* = is_H(\textbf{v}) \in \W$.  Then we define the current Nilsson's probability local structure $(S_\phi,\P(S_\phi),\mu_\phi)$ satisfying Kolmogorov axioms in Definition \ref{Def:p-space}, such that this local sample space $S_\phi$ is defined
  by
\begin{equation}\label{eq:IFOLpstructureLoc}
S_\phi =_{def} \{v:H_\phi\rightarrow X | ~~for~ every ~~A \in H_\phi \subset H, v(A) = \textbf{v}(A) ~~if~ \textbf{v}(A) \neq \bot\}
\end{equation}
so that for any $v\in S_\phi$, and ground atom $A\in H_\phi$ such that $\textbf{v}(A) = \bot$, we can have that $v(A)\in X$ can have the value different from $\bot$ as well.\\
So, analogously to the global constraints in Definition \ref{def:KIconstraint}, we introduce the local Probability Density Constraints:
for every atom $A \in H_\phi$ such that $u = \textbf{v}(A) \neq \bot$, the probability that $A$ has the truth-value $u$ must be equal to 1, i.e.,
\begin{equation}\label{eq:KIconstraintLoc}
\forall A\in H_\phi.(~if ~~u = \textbf{v}(A) \neq \bot ~~~~then ~~\sum_{v\in S_\phi, v(A) = u} KI_\phi(v) = 1)
\end{equation}
\end{definition}
Let us show that these local current Nilsson's probability structures has the analog important properties as the global symmetry:
\begin{coro}\label{coro:IFOLpstructure2}
Any local current Nilsson's probability structure with sample set $S_\phi$ in (\ref{eq:IFOLpstructureLoc}) satisfy the probability density constraints in (\ref{eq:KIconstraintLoc}).\\
 Thus, it is consistent with current robot's knowledge database $\K$ (in Definition \ref{Def:knowledge}), consistent with Kolmogorow probability axioms and logical deduction, and allows inductive reasoning and confirmation of universally quantified hypotheses.
\end{coro}
\textbf{Proof}: It is enough to show the consistency with current robot's knowledge database $\K$. The local symmetry transformation $S_\phi$ has no effects, from (\ref{eq:IFOLpstructureLoc}), for the ground atoms in $H$ which are not in $H_\phi$.
 Thus, for each ground atom $A \in H_\phi$ for which the current truth-value in $\K$ is $u = \textbf{v}(A) \neq \bot$, we have that  the probability density constraint (\ref{eq:KIconstraintLoc}) is satisfied. That is, the probability $p_u$ that this atom has the truth-value $u$ is equal to 1:

$p_u = \sum_{v\in S_\phi, v(A) = u} KI_\phi(v) $
$= \sum_{v\in S_\phi} KI_\phi(v)$ ~~~~~~~~~~~~~~from (\ref{eq:IFOLpstructureLoc})

$ = \mu_\phi(S)$
$ = 1$.\\
which preserves robot's current knowledge database and all logical deductions from them.
\\$\square$\\
So, from the locality, it holds that for any formula $\psi(\textbf{x})$ composed by \emph{only the predicates in} $P_\phi$ (defined by our local decision problem),  for any assignment $g$ and $v\in S_\phi$,
\begin{multline} \label{eq:esem3s7}
 \textbf{v}^*(w_N(\lessdot \psi(\textbf{x})\gtrdot^\beta/g, g(x_X),g(x_p) )) =\\ 
 =\left\{
    \begin{array}{ll}
 t, & \hbox{if $~~
 ~g(x_p) = \sum_{v\in S_\phi, v^*(\psi(\textbf{x})/g) =g(x_X)} KI(v) $}\\
      f, & \hbox{otherwise}
    \end{array}
  \right. ~~~~~~~~~~~~~~~~~~~~~
  \end{multline}
 as discussed after equation (\ref{eq:esem3s2}) and Example \ref{ex:events} also in the case of the sentences about temporal events (composed by the "atomic event" predicates as well).

That is, for any sentence $\psi(\textbf{x})/g$ defined by the subset of predicates $P_\phi$ (obtained by focusing on the concrete problem described by the initial sentence $\phi$), we are able to compute its probabilities to have a particular truth-value in $X$.
 As in the case of global symmetry, these true facts about the probabilities of the unknown sentences to have a particular truth-value can be then inserted in the conscious part of robot's cognitive system by (see (\ref{eq:esem2}) for more details),

  $Know(in ~presence, \textbf{I}, w_N(\lessdot \phi\gtrdot, u,p_u))$,\\
 to be used by robot's autoepistemic deductions and explanations as well.

Intensional knowledge  also treats propositions as objects of reasoning.
The robot can represent statements such as:

\emph{Robot knows that Door(room5) is probably closed with probability }$82\%$.\\
rather than merely storing the proposition "Door(room5) is closed."
This is an epistemic representation rather than a purely extensional one.

Self-awareness:
One of the novel ideas in this AGI proposal is a symbolic model of the robot itself. The robot maintains knowledge such as: current goals, executed actions,
internal beliefs, plans, uncertainty about its own knowledge.
This goes beyond classical planning systems that typically represent only the external world.

This paper provides a \emph{probabilistic extension}  of the many-valued typed $IFOL_B$  with its autoepistemic axioms and many-valued deduction presented in the paper \cite{Majk26} for neuro-symbolic AGI.
 It  is dedicated to show how this defined IFOL in \cite{Majk22} can be used for a new generation of intelligent robots, able to communicate with humans with this intensional FOL supporting the meaning of the words and their language compositions, heaving the  four-level neuro-symbolic cognitive structure of AGI robots \cite{Majk24a,Majk24ar} and in first section in \cite{Majk26} as well.

We argue that the example, used for the spatial natural sublanguage in \cite{Majk23r} and \cite{Majk24a}, can be extended in a similar way  to cover  more completely the rest of human natural language, and hence the method provided by this paper is a main theoretical and philosophical contribution to resolve the open problem of how we can implement the deductive power based on $IFOL_B$ for new models of robots heaving strong AI capacities.  Intensional FOL is able to represent the Intentional States (mental states such as beliefs, hopes, and desires), typical for human minds.

Despite the best efforts over the last years, deep learning is still easily fooled \cite{NYCl15}, that is, it remains very hard to make any guarantees about how the system will behave given data that departs from the training set statistics. Moreover, because deep learning does not learn causality, or generative models of hidden causes, it remains \emph{reactive}, bound by the data it was given to explore \cite{Marc18}.
In contrast,
we learn from our actively gathered sensorimotor experiences and form conceptual, loosely hierarchically structured, compositional generative predictive models.
By proposed four-level cognitive robot's structure, $IFOL_B$ allows robots to reflect on, reason about \emph{probabilistically as well}, anticipate, or simply imagine scenes, situations, and developments within in a highly flexible, compositional, that is, semantically meaningful manner. So, $IFOL_B$ enables the robots to actively infer highly flexible and adaptive goal-directed behavior under varying circumstances \cite{Russ20}.

With this integrated four-level robot's knowledge system presented in the Appendix B (Section 7), where the last level represents the robot's neuro system containing the neural networks to calculate maximum information entropy and the deep learning as well, we obtain that also the semantic theory of robot's $IFOL_B$ is a procedural one, according to which sense is an abstract, pre-linguistic procedure detailing what operations to apply to what procedural constituents to arrive at the product (if any) of the procedure.

In this research, specifically within the framework of Strong-AI Autoepistemic Robots, we describe how a robot uses $IFOL_B$ to treat its own internal software and motor routines as objects of thought.

Instead of a motor program being a "black box" that just runs, our framework allows the robot to reify (turn into a thing) the program. Here is the specific mechanism of how that labeling works:
\begin{enumerate}
  \item \textbf{The "Self" as the Coordinator};\\
   We define the robot's identity as a constant term in logic, usually denoted as \textbf{I}. This \textbf{I} represents the "Main Coordination Program".  When the robot performs an action, it isn't just "executing code"; it is asserting a logical relationship between itself and a specific sub-routine.\\
  \item \textbf{The Labeling Process: Reification}\\
To label a motor program (e.g., a routine that moves the right arm to pick up a block), the robot uses an Intensional Abstraction Operator.\\
- \textbf{The Program}: Let’s say the raw code for moving an arm is $P_{1}$ and motor programs be defined by true atoms of the binary predicate $MotPrg(x_1,x_2)$ where the variable $x_1 \in \V$ defines the row codes in neural system and $x_2 \in \V$ defines the labels (names) of these row codes, so that for this assignment $g:\V \rightarrow \D$ forb variables in $\V$, $g(x_1) = P_{1}$ and  $g(x_2) = "moving ~right~ arm"$ for which the ground atom $MotPrg(x_1,x_2)/g = MotPrg(P_1,moving ~right~ arm)$ is true.\\
- \textbf{The Logic Term}: The robot uses the abstraction operator $\lessdot\_ \gtrdot$ to create a "name" for this code. The label becomes an intensional entity - a symbol the robot can think about without actually running the code. In fact by transformation of this logic predicate into abstracted term, we obtain (by using the intensional mapping $I$) that

  $u= g^*(\lessdot MotPrg(x_1,x_2)\gtrdot^{x_1}_{x_2}) = I(MotPrg(g(x_1),x_2))$

  $= I(MotPrg(P_1,x_2)) \in D_1$ is an intensional entity (an unary concept) such that its extension (for extensionalization function $h$) is just a singleton, i.e., $h(u) = \{moving ~right~ arm\}$. Thus, the intensional entity $u = I(MotPrg(P_1,x_2))$ can be used as a label (name) for the raw code $P_1$.\\
  - \textbf{The Predicate}: The robot then uses a ternary predicate like $Exec(y_1,y_2, \lessdot \phi(\textbf{x})\gtrdot^\beta_\alpha)$ with $g(y_1) = "in~ present"$ and $g(y_2) = \textbf{I}$, such that  $Exec(y_1,y_2, \lessdot MotPrg(x_1,x_2)\gtrdot^{x_1}_{x_2})/g = Exec(in~ present, \textbf{I}, I(MotPrg(P_1,x_2)))$  translates to:

  "\emph{\textbf{I} (the coordinator) am currently executing the motor program of raw code} $P_1$." \\
  \item \textbf{Example: A "Self-Aware" Gripper}\\
In a specific scenario described in his work on Intensional Logic for Robots, a robot tasked with picking up a heavy object doesn't just experience a motor failure; it reasons about it:\\
- \textbf{System 1 (Neural)}: The sensors in the gripper detect high torque and "slippage".\\
- \textbf{The Mapping}: This sensory pattern is mapped to the intensional concept of \(\textit{Heavy}\).\\
- \textbf{The Labeling}: The robot identifies that its current motor program $Grasp(x)$  with $g(x) = "object"$ is failing,\\
- \textbf{Autoepistemic Reasoning}: The robot generates the logical statement  "\emph{I know that I cannot lift this object while I am executing the 'Grasp' routine}.":\\\\
$Know(in~ present, \textbf{I}, \lessdot \neg CanLift(\textbf{I},object) \wedge Exec(in~ present, \textbf{I}, \lessdot Grasp(x)\gtrdot_x)\gtrdot)$\\
$\square$
\end{enumerate}
%


\section{Some Confrontations and Conclusions}
The most, as far as I know, similar past symbolic frameworks to this one for intelligent general-purpose robots with the self-awareness (the systems that share three properties: explicit symbolic self-representation, formal meta-level reasoning and architectural (not emergent) self-modeling) are the Epistemic/modal logic agent systems and Metacognitive cognitive architectures (e.g., SOAR, ACT-R with meta-layer). However, both of them can be considered as formal ancestors of our robot's system especially the Epistemic/modal logic agent systems.

SOAR (developed in Carnegie Mellon University) instead is production-rule based, not intensional FOL, less focus on formal semantics and more cognitive engineering than logical ontology. SOAR \cite{LNRo87} is practically closer in implementation, but philosophically less rigorous.
Both ogf them lack the intensional semantics, grounding layer (neural/sensory) and bilattice logic for inconsistency.
Other differences:\\\\
\begin{tabular}{|c|c|c|c|}
  \hline
  &~~\textbf{IFOL-based}~~ & \textbf{Epistemic logic agent} & \textbf{SOAR}  \\
  \hline
  \textbf{Formal self symbol} & Yes & Yes & Imlicit \\
  \hline
  \textbf{Nested meta-knowledge} & Yes & Yes & Partial \\
  \hline
  \textbf{Logical consistency tracking} & Yes & Yes & Limited \\
  \hline
  \textbf{Neural grounding} & Yes & No & Partial \\
    \hline
   \textbf{Embedding of neural LLM} & Yes & No & No \\
  \hline
   \textbf{Many-valued Extension} & Yes & No & No \\
  \hline
   \textbf{Probabilistic Extension} & Yes & No & No \\
  \hline
  \textbf{Theoretical AGI claim} & Yes & No & No \\
  \hline
\end{tabular}
\\
\\

Large Language Models rely on statistical probability. They predict the "next most likely word" based on vast training data, which often results in plausible-sounding "hallucinations" rather than true logical deduction. Differently, Symbolic Logic uses formal rules, variables, and mathematical symbols to process truth-values. It guarantees determinism, meaning the same inputs will always produce the mathematically correct output without guessing.

Now we are able to make the compare the self-awareness of our approach with that used by current dominant only neural LLM-style  self-modeling: \textbf{logical self-awareness} and \textbf{LLM-style self-modeling} aim at similar surface behavior (talking about themselves, reasoning about their own knowledge), but they are fundamentally different in mechanism and ontology.
\begin{enumerate}
  \item \textbf{Nature of the Self-Model}:\\
  - \textbf{IFOL-based}: "Self" is explicit symbolic object inside a formal logic; Self-awareness is implemented via intensional logic with self-reference; Meta-knowledge is structurally encoded. The system has a formal self-entity in its ontology.\\\\
  -\textbf{LLM-based}: "Self" is implicit statistical pattern learned from data; there is no symbolic internal object corresponding to “self”: there is no stable internal identity token representing the agent. Self-awareness is behavioral and linguistic, not structural.
  \\
  \item \textbf{Meta-Reasoning}:\\
  - \textbf{IFOL-based}: Meta-reasoning is explicit (statements about statements are first-class citizens, logical inference rules apply at both object and meta levels, self-reflection is formally defined).\\\\
  -\textbf{LLM-based}: Meta-reasoning is emergent, pattern-based, often fragile (it does not internally store a verifiable belief state; it does not logically derive meta-knowledge; it may contradict itself across turns). So, LLM meta-cognition is simulated, not grounded in formal self-belief tracking.
  \\
  \item \textbf{Grounding}:\\
  - \textbf{IFOL-based}: Self-model must be grounded in: neural perception, robot embodiment, sensory interaction. The symbolic “self” connects to physical experience. The concepts are natural language expressions, so can use also LLM vector/tensor representations. The architecture: Neural layer $\mapsto$ Conceptual layer $\mapsto$ Logical layer.\\\\
  -\textbf{LLM-based}: Self-model is grounded in: text corpora and human discourse about minds and AI (no sensory grounding, no embodiment, no persistent physical self).
  \\
  \item \textbf{Stability of Identity}:\\
  - \textbf{IFOL-based}: Identity is persistent (there is a defined agent symbol, the system maintains structured knowledge about itself, contradictions can be tracked using bilattice logic). Identity is architecturally enforced.\\\\
  -\textbf{LLM-based}:  Identity is: session-dependent, prompt-dependent, sometimes inconsistent. There is no persistent internal agent representation across contexts unless externally scaffolded.
  \\
  \item \textbf{Philosophical Interpretation}:\\
  - \textbf{IFOL-based}: Founded on Higher-Order Thought theory, formal epistemic self-reference and computational metacognition. \\
  - \textbf{LLM-based}:  Founded on  predictive processing of self-talk, linguistic simulation of agency and behaviorist self-description.
  \\
  \item \textbf{Type of Self-Awareness}:\\
  \begin{tabular}{|c|c|c|}
    \hline
     & ~~\textbf{IFOL-based}~~ & ~~\textbf{LLM-based}~~ \\
     \hline
    \textbf{Explicit self symbol} & Yes & No \\
    \hline
    \textbf{Formal meta-logic} & Yes & No \\
    \hline
    \textbf{Grounded in perception} & Yes & No \\
    \hline
    \textbf{Persistent belief tracking} & Yes & No \\
    \hline
    \textbf{Axiom's-based security} & Yes & No \\
    \hline
    \textbf{Separation consciousness/unconsciousness} & Yes & No \\
    \hline
    \textbf{Phenomenal consciousness} & No & No \\
    \hline
  \end{tabular}
 \end{enumerate}
 .
\\
So, in our vision where neural LLM can be used as significant language and common sense learning for robots (by considering the intensional concepts of $IFOL_B$ are based on the words (tokens) and possible various knowledge relationships like ISA, PART-OF, etc., between these concepts), $IFOL_B$-based strong-AI robot system is a definitely a \emph{significant extension of the \emph{LLM}} and not its concurrent. The core distinction  is that the self-awareness in $IFOL_B$ is an architectural property while in LLM is a statistical behavior. Our model is more principled, more structurally coherent but still theoretical, while LLM is highly capable in self-description, practically impressive, but lack formal self-model consistency.
\\
\textbf{Remark:}
Thus, for $IFOL_B$ model, by using LLM as a neural part of robot's neuro-symbolic paradigm, we obtain a much impressive strong-AI robots, able to
use LLMs for generative reasoning and natural language manipulation, to use symbolic logic for belief tracking and meta-consistency and to add grounding layer. That would combine $IFOL_B$ structural rigor and LLM’s flexible reasoning. This is actually where modern \emph{neuro-symbolic \emph{AGI } research is heading}.
\\$\square$\\
To bridge the gap between probabilistic text and deterministic reasoning, researchers combine LLMs and symbolic systems primarily through neuro-symbolic frameworks. This typically involves:
\begin{itemize}
  \item Translation and Formulation: You can use the strong natural language capabilities of an LLM to translate a complex, real-world problem into a formal symbolic syntax of $IFOL_B$. And viceversa.\\
  \item Execution via Symbolic Solvers: Once translated, the problem is handed over to a deterministic many-valued deductive system of $IFOL_B$  which performs error-free logical inference \cite{Majk26}.
\end{itemize}
This Neuro-Symbolic Framework stands out in the current Neuro-Symbolic (NeSy) AGI landscape because of its unique focus on safety guarantees, autoepistemic reasoning (how a robot reasons about its own ignorance), and non-binary logic.
While other frameworks integrate neural networks with symbolic logic to improve data efficiency or accuracy, this framework specifically addresses the unpredictability of a robot acting in the physical world.\\
\textbf{Key Benefits and Applications}:
This framework brings several distinct advantages to neuro-symbolic AI and advanced robotics:
\begin{itemize}
  \item \textbf{Safe Logic Deductions}: By using formal axioms, engineers can strictly control and predict a robot's reasoning. This guarantees that even if a neural network outputs messy statistical data, the robot’s final physical actions remain logically bounded and secure.
  \item \textbf{Human-like Epistemic Causality}: It allows robots to mimic human intelligence by converting statistical patterns into clear, directional cause-and-effect rules (logic entailments).
  \item \textbf{Gradual Learning Expansion}: As the robot experiences new things, the boundaries of its CKA dynamically shift. Facts seamlessly move from the "Unknown" bucket into "True" or "False" without breaking the underlying database or safety rules.
\end{itemize}
The direct comparison below shows how this architecture stacks up against other leading neuro-symbolic research frameworks:
\begin{enumerate}
  \item This framework:\\
  - \textbf{Core Approach}: Hybrid/Autoepistemic (Neural networks feed continuous data into an Intensional First-Order Logic (IFOL) structure) with LLM neural structures as well for self-process learning and natural language communications.\\
  - \textbf{Type}: 4-Valued Logic (True, False, Unknown, Inconsistent) via Closed Knowledge Assumption.\\
  - \textbf{Strength}: Strict controlled security and deterministic safety boundaries. Calculates  probabilities for unknown logical outcomes\\
  - \textbf{Application}:  Physical Robotics and Strong-AI Agents.
    \item Logical Neural Networks (LNN) framework:\\
  - \textbf{Core Approach}: Integrative (Neurons directly represent logical operations (AND, OR) inside the neural net).\\
  - \textbf{Type}: Bound-based real values (probabilities/intervals).\\
  - \textbf{Strength}: End-to-end differentiable; can optimize logic parameters using standard AI backpropagation.\\
  - \textbf{Application}:  Domain-specific reasoning and Knowledge Graphs.
  \item DeepProbLog framework (more is provided in Appendix C (in Section 8)):\\
  - \textbf{Core Approach}: Hybrid (Plugs neural networks directly into a probabilistic logic programming engine).\\
  - \textbf{Type}: Probabilistic Logic (handling likelihoods, not contradictions).\\
  - \textbf{Strength}: Calculates  probabilities for complex logical outcomes.\\
  - \textbf{Application}:  Visual question answering and neuro-symbolic games.
  \item Differentiable Inductive Logic Programming (DILP) framework:\\
  - \textbf{Core Approach}: Integrative (Learns explicit, human-readable logical rules from raw data).\\
  - \textbf{Type}: Standard Binary Logic handled smoothly through calculus gradients.\\
  - \textbf{Strength}: High data-efficiency; learns broad rules from only a few human examples.\\
  - \textbf{Application}:  Rule induction and automated software programming.
  \item LLM + Symbolic Solver (LLM-SS) framework:\\
  - \textbf{Core Approach}: Modular/Agentic (A Large Language Model acts as the interface and translates natural language into standard computer code for a hardcoded symbolic solver).\\
  - \textbf{Type}: Classic Binary Logic or Constraint-Satisfaction.\\
  - \textbf{Strength}: High  natural language flexibility and conversational fluency.\\
  - \textbf{Application}:  Mathematical problem solving and multi-agent systems.
\end{enumerate}
Useful integration for our framework (point 1 above) can be done by extending it with methods of rule induction from row data used in DILP framework (point 4 above) because our framework provides the LLM component, probabilistic reasoning explained in this paper and parsing of natural language into conceptual structures of $IFOL_B$ and corresponding syntax of First-Order Logic formulae (analog to LLM-SS). The rule induction can produce the self-learned rules represented by logical implication and thus used by autoepistemic deduction process of $IFOL_B$ provided in \cite{Majk26},
 by preserving
 the following Key Structural Differences with the other frameworks presented above:
\begin{enumerate}
  \item Handling "Ignorance" and Data Contradictions:\\
  -\textbf{The Norm}:  Frameworks like DeepProbLog or LNN use probability bounds (e.g., "70 percent true"). If a neural network gives contradictory data, these systems often experience mathematical breakdown or confidently average out the conflicting information \cite{Chen25}.\\
  - \textbf{Our Advantage}: By using the Closed Knowledge Assumption (CKA), this framework natively flags data conflicts as "Inconsistent" and missing data as "Unknown". The robot acknowledges its own confusion instead of generating a broken or dangerous action.
  \item The Path to AGI, Language vs. Physical Embodiment:\\
  -\textbf{The Norm}:   LLM-Symbolic hybrid approaches use massive language models to reason. They excel at domain-agnostic benchmarks but lack true symbol grounding-meaning they do not understand how words relate to physical objects \cite{Chen25}.\\
  - \textbf{Our Advantage}:  This system uses Intensional First-Order Logic (IFOL). It bridges natural language directly with a robot's real-time sensorimotor experiences. The symbols are actively constructed by the robot's physical interactions, making it highly effective for embodied strong AI.
  \item Strict Safety Guarantees:\\
  -\textbf{The Norm}:   In integrative approaches like LNN, symbolic reasoning happens directly inside the neural network. While elegant, neural networks can still experience edge cases and hallucinate or produce unpredictable outputs \cite{Chen25}.\\
  - \textbf{Our Advantage}:   It acts as a hard security firewall. The neural network handles messy real-world perception, but the separate IFOL symbolic component possesses final veto power. An action is only executed if it mathematically passes deterministic safety axioms.
\end{enumerate}
From a research perspective, this probabilistic neuro-symbolic robot is best viewed as a \emph{semantically grounded cognitive architecture} a logically rigorous framework that integrates: neural perception (conceptually), symbolic knowledge representation, temporal reasoning,probabilistic inference, intensional semantics,
explicit self-representation.
So, we would be able to develop the interactive robots which learn and understand spoken language via multisensory grounding and internal robotic embodiment. Endowed with suitable $IFOL_B$ information-processing biases, the robot's AI may develop that will be able to explain the reality it is confronted with, reason about it also probabilistically, and find adaptive solutions, making it Strong AI.
%
%

%

\section{Appendix: From Statistical Physics to Probabilistic $IFOL_B$ \label{sec:Helmholtz}}
In statistical physics, free energy is the bridge between the microscopic world of atoms and macroscopic thermodynamics. It represents the portion of total energy available to do work and acts as a generating function from which all physical properties (such as pressure, heat capacity, and magnetization) can be derived.

The foundational link between statistical mechanics and thermodynamics is the partition function, denoted as $\Z$. It is the sum of all possible Boltzmann factors over all microstates, formally defined as
$$\Z = \sum_i \mathrm{e}^{-\beta E_i}$$
 where $E_i$ is total energy of the microstate $i$, and $\beta = \frac{1}{k_B T}$ is the thermodynamic inverse temperature. The \emph{free energy} of a system is defined directly from the partition function via the following logarithmic relationship:
 $$F = - \frac{1}{\beta}ln \Z = - k_BT ln \Z$$
 Depending on the thermodynamic constraints (which variables are held constant), different statistical ensembles yield different free energies. In what follows we will use the Helmholtz free energy, used in the canonical ensemble where temperature $T$, Volume $V$, and the number of atoms $N$ are held constant.  such that the free energy reaches its minimum at thermal equilibrium.
 In Helmholtz case, the free energy is defined by
 \begin{equation}\label{eq:Helm}
 F = U - TE
 \end{equation}
 where $U$ is internal energy of system and $E$ is the entropy. So that the minimal free energy is obtained in the thermal equilibrium (the constant temperature $T$) with minimal $U$ and maximal entropy $E$.

 This is just our AGI case, because in $IFOL_B$ we have the constant number of (ground) atoms in Herbrand base $H$, and the constant "volume" represented by the  sample set $S$ in (\ref{eq:IFOLpstructure}), and "minimal internal energy" $U$ is obtained by the atomic knowledge database $\K$ based on the CKA, while the maximal entropy $E$ which defines the probability density function $KI:S \rightarrow X$ is defined by Shannon's principle for maximal information entropy in (\ref{eq:KIentropy}).

 The probabilistic expansion of knowledge database $\K$, as result of introduction of the probability density function $KI$, and transformation generated by introduction of Nilsson's probability structure $(S_i,\P(S_i), \mu_i)$ of global/local symmetry in \emph{current} many-valued model $\textbf{v}^*:\L_0 \rightarrow X$, is defined for all sentences in $\L_0$  by:
 \begin{equation}\label{eq:Helm2}
 \K_i^+ = \{w_N(\lessdot\phi\gtrdot,u,p_u)~|~\phi\in \L_0, \textbf{v}^*(\phi) = \bot, u \in X~~and ~p_u =\sum_{v\in S_i, v^*(\phi) = u}KI(v)\}
 \end{equation}
 Consequently, for AGI, the minimal total "internal energy" of robot's cognitive system  is its \emph{internal knowledge}
 \begin{equation}\label{eq:Helm32}
 U = \K \bigcup \K_i^+
\end{equation}
in this "thermal equilibrium" obtained after introduction of probability density function $KI$, during which $T$ is a constant, representing the transformation of the maximal information entropy $E$ (which defines function $KI$) into additional probabilistic knowledge $\K_i^+$, as can be shown from (\ref{eq:Helm}),

$\K = F = U- TE$ 
$ = (\K + \K_i^+) - TE$ ~~~~~~~~~~~~from (\ref{eq:Helm32}) \\
so that
$$ \K_i^+ = T E$$
and hence, this analogy between statistical physics and $IFOL_B$ AGI, can be summarized by the following table:\\\\
\begin{tabular}{|c|c|c|}
  \hline
  &&\\
    & \textbf{Statistical Physics} & \textbf{AGI} \\
  \hline
  &&\\
  \textbf{Maximal entropy} & $E$ & $E= max_{KI}(- \sum_{v\in S}KI(v) \cdot logKI(v))$ \\
   &&\\
   \hline
   &&\\
   \textbf{Minimal internal} & U & $U =\K + \K_i^+$ \\
    \textbf{energy}&&\\
   \hline
   &&\\
  T & Temperature & $T =  \frac{\K_i^+}{E}$, ~~~i.e. transformation \\
  &&$~~T:E~\mapsto ~ \K_i^+$\\
  \hline
  &&\\
  \textbf{Helmholtz } & $F = U-TE $ & $F = \K~~~~~~~~$ Knowledge DB (CKA)\\
  \textbf{free energy}&&\\
  \hline
\end{tabular}
 .\\\\

In thermodynamics, free energy is a measure of the amount of
work that can be extracted from a physical system. And, from above, in AGI, free energy is a measure of the amount of
knowledge $\K$ that can be extracted from a symbolic logic system. \\
It is interesting how in AGI the knowledge takes the place of the physical energy: in effect, in the physics the energy is mean that can be used to produce the useful physical work, while in the abstract cognitive space, the knowledge is the mean from which we can obtain useful deductions (a kind of intellectual work).
\\\\
There are interesting similarities with Karl Friston's Active Inference:\\\\
\begin{tabular}{|c|c|}
  \hline
  &\\
     \textbf{Active Inference} & $\text{IFOL}_{B}$ \textbf{Neuro-symbolic Approach} \\
  \hline
  &\\
  probabilistic beliefs &probabilistic logic\\
   \hline
   &\\
  generative model&   symbolic knowledge base\\
   \hline
   &\\
   prediction error& logical inference\\
  \hline
  &\\
  Bayesian updates & temporal probabilistic reasoning\\
  \hline
  &\\
  latent states & intensional knowledge\\
 \hline
\end{tabular}
.\\\\\\
\textbf{The symmetry group }$G_\K$ \textbf{of knowledge transformations}:\\\\
Consequently, we can introduce a formal definition of the group of knowledge transformations that preserves the current knowledge database $\K$ (free Helmholtz energy in the first table above):
\begin{definition} The symmetry group $G_\K = (\circ, e_i, i\geq 0)$ of knowledge transformations\label{sec:Kgroup}, for a fixed current world $\textbf{I}_B^* = is_H(\textbf{v}) \in \W$ of current Herbrand interpretation $\textbf{v}:H \rightarrow X$, is composed by the executed robot's transformations   $e_i =  (S_i,\P(S_i),\mu_i)$, $i\geq 1$, equal to introduced Nilsson's probabilistic structures for global/local probabilistic knowledge extensions,
 $$e_i: \K \mapsto \K \bigcup \K_i^+$$
  with identity element  with empty sample space $e_0 =(\emptyset,\emptyset, \mu_\emptyset): \K \mapsto \K$  produces empty probabilistic knowledge  $\K_{\emptyset}^+= \emptyset$.
\\
  The binary operation of the group "$\circ$", for any two transformations $e_i$ and $e_k$, is defined by
  \begin{equation} \label{eq:groupop}
e_n = e_i\circ e_k =
 \left\{
    \begin{array}{ll}
 (S_i\biguplus S_k, \P(S_i)\biguplus\P(S_k), \mu_i\biguplus\mu_k), & \hbox{if $~~i\neq k$}\\
      e_k, & \hbox{otherwise}
    \end{array}
  \right.
\end{equation}
where $\biguplus$ is disjoint union, so that for $Y \in  \P(S_i)\biguplus\P(S_k)$,
\begin{equation} \label{eq:groupop}
(\mu_i\biguplus\mu_k)(Y) =
 \left\{
    \begin{array}{ll}
 \mu_i(Y), & \hbox{if $~~Y \in \P(S_i)$}\\
      \mu_k(Y), & \hbox{otherwise}
    \end{array}
  \right.
\end{equation}
So, this resulting knowledge transformation is represented by
  $$e_n = e_i\circ e_k: \K \mapsto \K \bigcup (\K_i^+\bigcup\K_k^+)$$
and hence $e_n$ produces the cumulative probabilistic knowledge $\K_i^+\bigcup\K_k^+$.
\end{definition}
The global symmetry group of transformations, for given current Herbrand interpretation $\textbf{v}:H \rightarrow X$, has only two elements: the identity element $e_0$ and the unique element of global knowledge transformation in this possible world, defined by (\ref{eq:IFOLpstructure}) in Definition \ref{def:mvR}.

For local symmetry transformations, each time that robot needs to compute the probability for unknown facts (to have some true value), necessary for some real-time decisions, a new element will be added to this group $G_\K$, and the probability knowledge will be cumulatively incremented.

Any change of possible world, thus of  Herband interpretation (changing $\K$ by learning (or by logical inference) the truth-value of some ground atom in Herbrand base $H$,  or by extending atomic knowledge database $\K$ with a new ground atom (which had previously unknown value $\bot$) which obtained some logical value in $\{f,t,\top\}$,
deletes all extended probabilistic knowledge obtained in previous knowledge database $\K$, \\
So, in this new current world of knowledge, robot will generate by symmetry transformations (global or local) the new symmetry group $G_\K$ for this new knowledge database $\K$.
Consequently, any current atomic knowledge database $\K$ \emph{remains invariant} under the group of transformations in $G_\K$.

\newpage
\section{Appendix B: Neuro-symbolic Archetuture}
. \\\\\\
 \begin{figure}
$\vspace*{-44mm}$
\centering{
 \includegraphics[scale= 1.0]{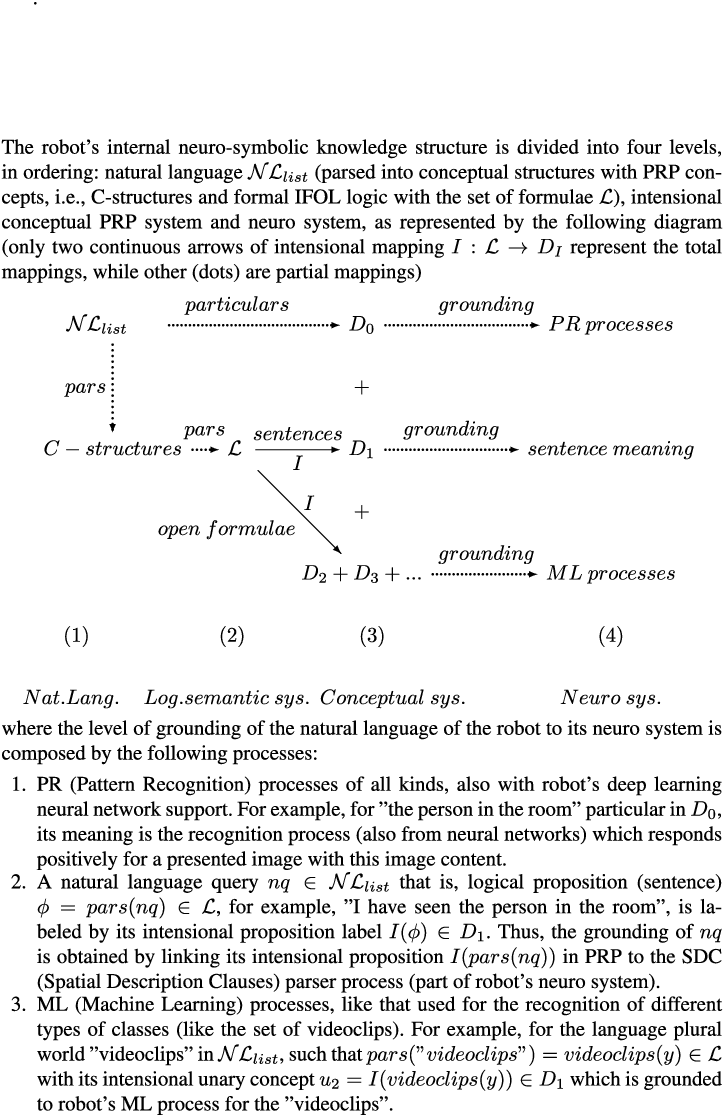} }
  $\vspace*{-11mm}$
 \end{figure}
\newpage
\section{Appendix C: Comparation with DeepProbLog \label{sec:Helmholtz}}
The previous version of neuro-symbolic system developed in last ten years, is pure probabilistic model obtained by extension of the ProbLog, developed 20 years ago, (De Raedt et al., \cite{LuKH07} symbolic logic system by neural networks used to compute the probabilities of the ground atoms \cite{MDKL18}. However, the ProbLog is based on the  2-valued predicate logic Prolog (\emph{Programming} Logic Language, used by programmers to develop different kinds of programs in AI applications in that time), which is not a good candidate for human-like modern AGI for different reasons: the only 2-valued truth-based semantics and computability limits. This choice of course is made more than 30 years ago, when we had only standard two-valued FOL and its logic subsystems (restrictions).

In effect, the authors of DeepProbLob, described their neuro-symbolic system as follows \cite{MDKL18}:

 "\emph{We introduce DeepProbLog which has a unique set of features: (i) it is a programming language that supports neural networks and machine learning, and it has a well-defined semantics (as an extension of Prolog, it is Turing equivalent); (ii) it integrates logical reasoning with neural networks; so both symbolic and subsymbolic representations and inference; (iii) it integrates probabilistic modeling, programming and reasoning with neural networks (as DeepProbLog extends the probabilistic programming language ProbLog, which can be regarded as a very expressive directed graphical modeling language [De Raedt et al.,\cite{LKNP16}); (iv) it can be used to learn a wide range of probabilistic logical neural models from examples, including inductive programming....\\
 \\
 We start from an existing probabilistic logic programming language, ProbLog (De Raedt et al., \cite{LuKH07}), and extend it with the capability to process neural predicates. The idea is simple: in a probabilistic logic, atomic expressions of the form $q(t_1, ..., t_n)$ (aka tuples in a relational database) have a probability $p$. Consequently, the output of neural network components can be encapsulated in the form of "neural" predicates as long as the output of the neural network on an atomic expression can be interpreted as a probability. This simple idea is appealing as it allows us to retain all the essential components of the ProbLog language: the semantics, the inference mechanism, as well as the implementation.}"

 So, this long-time research for DeepProbLog, in the same period of my long.time resarch is very different: I worked with Prolog in the period of my PhD, 1994-1998, and considered that we need more expressive logic than this 2-valued logic programming, for development of AGI independent robots. So. from 2001 I have spend 21 years to define this new more human-like Intensional FOL, published in 2022 \cite{Majk22} working alone for this ambitious objectives without any grants and founds, and only after that to begin the foundation work for AGI robotics.

 Thus, just these facts explains why their neural-symbolic AGI is substantially different from my recent proposal (from 2023).
 My approach and DeepProbLog fundamentally differ in architectural direction, logic structures, and cognitive scale:   DeepProbLog operates as a hybrid framework where a neural network feeds probabilistic facts directly into a separate logic solver\footnote{In Prolog, falsity is not an explicit truth value but is instead represented by the failure of a goal to succeed. Because Prolog operates on the Closed World Assumption, anything it cannot logically prove to be true (or which explicitly fails) is evaluated as false. Which is very different in $\text{IFOL}_{B}$, where we have also other important explicit truth values, for unknown and inconsistent information and falsity is the first-class concept like the truth.}. In contrast, this framework functions as an integrative system, embedding Intensional First-Order Logic ($\text{IFOL}_{B}$) alongside Shannon's maximum entropy neural engines to construct a unified architecture specifically optimized for autonomous AGI robotics.

 \textbf{Architectural and Logical Comparison:}\\
 \begin{tabular}{|c|c|c|}
  \hline
      & $\text{IFOL}_{B}$ \textbf{Neuro-symbolic Approach} &\textbf{DeepProbLog} \\
  \hline
    \textbf{Primary Goal} & Self-reflective AGI robotics  & Data-efficient neuro-symbolic tasks\\
     &(Self vs. environment)& (e.g., MNIST addition)\\
     \hline
      \textbf{Logic Base} & Intensional First-Order Logic  & Extensional Probabilistic Prolog  \\
      &($\text{IFOL}_{B}$)&(ProbLog)\\
    \hline
     \textbf{Truth Values} & 4-valued Belnap Bilattice & Continuous probabilities \\
         &$\{f,t,\bot,\top\}$ & over 2-valued (\([0,1]\))\\
    \hline
     \textbf{Uncertainty} & Nilsson's structure/Shannon Entropy & Distribution Semantics \\
     \hline
     \textbf{Scalability } & Local/Global Symmetry  & Knowledge Compilation\\
     \textbf{Mechanism} &Transformations&  (d-DNNFs, sdd)\\
     \hline
     \textbf{Key Capability} & Formal Self-Reference  & End-to-end backpropagation\\
      &and Autoepistemic Reasoning&  through logic\\
    \hline
    \textbf{Human Brain} & Free Energy   & No\\
     \textbf{Compatibility} &Helmholtz/ Karl Friston&  \\
    \hline
   \end{tabular}
   .\\\\
   \textbf{Three Core Structural Differences}:
   \begin{enumerate}
   \item  \textbf{Managing Contradictions vs. Strict Probabilities}:\\
   - \textbf{DeepProbLog} assigns a probability to a fact (e.g., 0.7::smokes(person)). If sensor data yields flatly conflicting conclusions, it dilutes the probability distribution. This can lead to logic bottlenecks or failure to compile under intense real-world operational noise.\\
   - $\text{IFOL}_{B}$ uses Belnap’s 4-valued logic to accommodate contradictions explicitly. If a robot receives conflicting data, the state is categorized under \(\top \) (Inconsistent) or \(\bot \) (Unknown). It handles these as distinct semantic values rather than failing computationally, preserving system stability.
   \item \textbf{Intensional vs. Extensional Representation}:\\
   - \textbf{DeepProbLog} relies on classical extensional logic. It treats predicates as sets of objects, making it difficult for an agent to separate a real concept from its physical instance.\\
   -  $\text{IFOL}_{B}$ applies Intensional Logic. This allows an AGI robot to reason about abstract properties and its own internal mental state. The robot can explicitly represent and evaluate the mathematical concept of "the Self" as distinct from external physical structures.
   \item \textbf{Execution Bottlenecks and Robotics Scaling}:\\
   - \textbf{DeepProbLog} relies on algebraic operators that scale poorly during complex tasks, generating severe training slowdowns when processing highly complex data.\\
   - $\text{IFOL}_{B}$  mitigates this scaling bottleneck by pairing neural networks with \emph{Local Symmetry Transformations}. When a robot needs to make a split-second physical decision, it isolates a small subset of relevant predicates. This prevents the system from having to compute entire world-states simultaneously, facilitating real-time physical navigation.
   \end{enumerate}
\end{document}